\documentclass{article}

\usepackage[numbers]{natbib} 
\usepackage{float}
\usepackage{bm} 
\usepackage{algorithm}
\usepackage{algorithmicx}
\usepackage{algpseudocode} 
\usepackage{amsfonts}
\usepackage{bm}
\usepackage{titletoc}

\usepackage{graphicx} 
\usepackage{amsmath}
\usepackage{array}
\usepackage{multirow}
\usepackage{subfigure}
     \usepackage[preprint]{neurips_2025}

\usepackage[utf8]{inputenc} 
\usepackage[T1]{fontenc}    
\usepackage{hyperref}       
\usepackage{url}            
\usepackage{booktabs}       
\usepackage{amsfonts}       
\usepackage{nicefrac}       
\usepackage{microtype}      
\usepackage{xcolor}         
\usepackage{amsthm}

\title{LHT: Statistically-Driven Oblique Decision Trees for Interpretable Classification}

\author{%
  Hongyi Li \\
  Department of Automation\\
  Harbin Institute of Technology, Shenzhen\\
  \texttt{23b904015@stu.hit.edu.cn} \\
   \And
   Jun Xu \thanks{Corresponding author}\\
    Department of Automation\\
   Harbin Institute of Technology, Shenzhen\\
   \texttt{xujunqgy@hit.edu.cn} \\
   \AND
   William Ward Armstrong \\
   Department of Computing Science \\
   University of Alberta \\
   \texttt{warmstro@ualberta.ca} \\
}

\begin{document}

\maketitle

\newtheorem{lemma}{Lemma}
\newtheorem{remark}{Remark}
\newtheorem{theorem}{Theorem}
\newtheorem{assumption}{Assumption}

\begin{abstract}
We introduce the Learning Hyperplane Tree (LHT), a novel oblique decision tree model designed for expressive and interpretable classification. LHT fundamentally distinguishes itself through a \emph{non-iterative, statistically-driven} approach to constructing splitting hyperplanes.Unlike methods that rely on iterative optimization or heuristics, LHT directly computes the hyperplane parameters, which are derived from feature weights based on the differences in feature expectations between classes within each node. This deterministic mechanism enables a direct and well-defined hyperplane construction process. Predictions leverage a unique piecewise linear membership function within leaf nodes, obtained via local least-squares fitting. We formally analyze the convergence of the LHT splitting process, ensuring that each split yields meaningful, non-empty partitions. Furthermore, we establish that the time complexity for building an LHT up to depth $d$ is $O(mnd)$, demonstrating the practical feasibility of constructing trees with powerful oblique splits using this methodology. The explicit feature weighting at each split provides inherent interpretability. Experimental results on benchmark datasets demonstrate LHT's competitive accuracy, positioning it as a practical, theoretically grounded, and interpretable alternative in the landscape of tree-based models. The implementation of the proposed method is available at \url{https://github.com/Hongyi-Li-sz/LHT_model}.
\end{abstract}

\section{Introduction}

Tree-based models, particularly gradient-boosted variants like XGBoost \cite{chen2016XGBoost}, consistently achieve state-of-the-art (SOTA) performance on tabular data tasks, often surpassing deep learning methods in this domain \cite{shwartz-ziv2021tabular, grinsztajn2022tree}. However, the expressive power of traditional decision trees is fundamentally limited by their reliance on axis-parallel splits \cite{NEURIPS2021_285a25c1}, which are often inadequate for capturing the complex dependencies and interactions among features. Oblique decision trees address this limitation by employing hyperplane splits based on linear combinations of features, offering substantially greater flexibility and modeling capacity. Yet, many existing tree algorithms depend on iterative optimization procedures or heuristic search methods to find suitable hyperplanes  \cite{marton2024gradtree, zharmagambetov2021non}, which can present challenges in terms of their underlying mechanics or applicability.

To offer an alternative approach, we introduce the Learning Hyperplane Tree (LHT), a novel oblique decision tree model characterized by its unique hyperplane construction methodology. Instead of relying on iterative optimization, LHT employs a \emph{statistical, non-iterative} approach to directly construct splitting hyperplanes at each node. Specifically, LHT derives these hyperplanes by leveraging local data statistics, primarily exploiting the differences in feature expectations between target and non-target classes within the current node's data subset. This statistically-driven mechanism allows LHT to capture complex feature relationships relevant for splitting without resorting to iterative refinement processes. LHT then recursively partitions the data using these oblique hyperplanes, progressively creating a refined partitioning of the feature space. To our knowledge, this specific mechanism—using feature expectation differences to directly guide the construction of oblique hyperplanes without iteration—is unique to LHT, offering a novel paradigm for building oblique decision trees. Furthermore, at its leaf nodes, LHT utilizes a locally fitted, piecewise linear membership function for prediction. The explicit calculation of feature weights contributing to each split also endows the model with inherent interpretability. 

Our main contributions are:
1) We propose LHT, a novel oblique decision tree that constructs hyperplanes directly using a distinctive statistical, non-iterative approach \textcolor{black}{and possesses universal approximation capability}, setting it apart from conventional optimization-based methods.
2) We detail the LHT training procedure, highlighting its statistically-driven feature weighting and threshold selection strategy for hyperplane definition, and provide analysis regarding its convergence properties.
3) We discuss LHT's inherent interpretability, which stems from the explicit feature contributions calculated at each splitting node.
4) We conduct extensive experiments on benchmark datasets, demonstrating that LHT achieves competitive performance compared to SOTA tree-based models in terms of classification accuracy.

\section{Learning Hyperplane Tree}
\label{ss2}
\begin{figure*}[ht]
\vskip -0.1in
\begin{center}
\centerline{\includegraphics[width=0.8\textwidth]{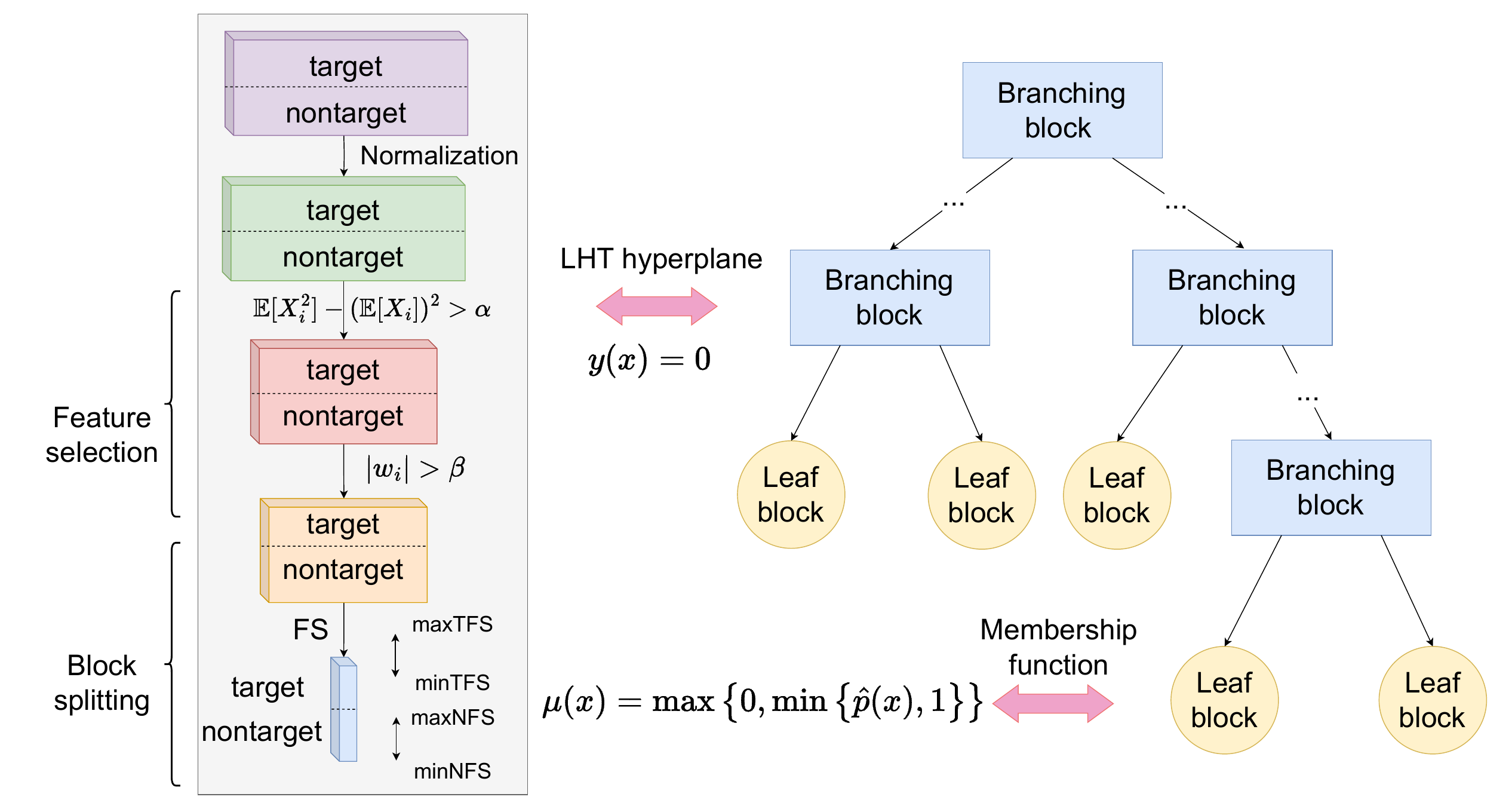}}
\caption{
The structure of LHT is illustrated. LHT consists of two types of blocks: branching blocks, which employ hyperplanes for sample partitioning, and leaf blocks, where least-squares fitted membership functions are used for classifying test samples. The construction of the LHT hyperplane consists of feature selection and block splitting.}

\label{LHT}
\end{center}
\vskip -0.2in
\end{figure*}

In this paper, each LHT is designed to solve a binary classification problem. For a multi-class classification task, an individual LHT is constructed for each class. For each class, the LHT treats the samples of that class as target samples and all others as non-target samples, which allows each LHT to tackle a binary classification problem specific to its assigned class. As a result, LHTs for different classes are tailored to handle distinct binary classification tasks. By combining the outputs of all LHTs, the overall multi-class classification problem can be effectively solved. 

As shown in Figure \ref{LHT}, an LHT begins with a root block and splits into a hierarchical structure. It consists of two types of blocks: branching blocks, which define hyperplanes to partition the sample data into subsets, and leaf blocks, which apply membership functions derived from local least squares fitting to assess a sample's membership in the target class.

\subsection{LHT Hyperplane}

The LHT hyperplanes tend to separate target samples from non-target samples using linear decision boundaries. Each branching block in the LHT has a hyperplane that partitions the data in the block into two subsets. These subsets are then assigned to the two subblocks resulting from the split. 
Typically, both sides of the split contain a mixture of both classes. However, the hyperplane can be chosen such that one side contains only target samples or only non-target samples.
The subblock that contains pure samples \textcolor{black}{or too few samples} is labeled as a leaf block and will not be split further. The subblock with mixed samples will continue to be split.
The data partitioning of a branching block consists of two steps: feature selection and block splitting, shown in Figure \ref{LHT}. \textcolor{black}{Section \ref{feature1} provides the details of feature selection, while Section \ref{block2} describes the branching block splitting process in detail.}

\subsubsection{Feature Selection}
\label{feature1}

Consider a branching block with $n$ samples, each described by $m$ features. Let $\mathbb{E}[X_i]$ denote the sample mean of the $i$-th feature across all $n$ samples. Features are selected by retaining those with high variance, satisfying the condition:
\begin{equation}
\mathbb{E}[X_i^2] - (\mathbb{E}[X_i])^2 > \alpha, \quad i = 1, 2, \dots, m,
\end{equation}
where $\mathbb{E}[X_i^2]$ is the second moment of the $i$-th feature, and $\alpha \geq 0$ is a threshold controlling the retention of informative features.

Within the block, define $\mathbb{E}[X_i^{\text{t}}]$ as the sample mean of the $i$-th feature for target samples and $\mathbb{E}[X_i^{\text{nt}}]$ for non-target samples. The separation degree $\text{SD}_i$ is given by:
\begin{equation}
\text{SD}_i = \mathbb{E}[X_i^{\text{t}}] - \mathbb{E}[X_i^{\text{nt}}], \quad i = 1, 2, \dots, m,
\label{SD}
\end{equation}
representing the mean difference between classes for feature $i$. A larger $|\text{SD}_i|$ indicates greater discriminative power. The maximum separation across features is:
\begin{equation}
\overline{\text{SD}} = \max \{ |\text{SD}_1|, |\text{SD}_2|, \dots, |\text{SD}_m| \},
\label{absSD}
\end{equation}
and the normalized weight of feature $i$ is:
\begin{equation}
w_i = \frac{\text{SD}_i}{\overline{\text{SD}}}, \quad i = 1, 2, \dots, m,
\label{feature_weight}
\end{equation}
where $-1 \leq w_i \leq 1$. Features with $|w_i| \geq \beta$, for $0 \leq \beta \leq 1$, are selected, with larger $\beta$ values yielding fewer but more discriminative features, and smaller $\beta$ values retaining more features.

\subsubsection{Block Splitting}
\label{block2}

For an input vector $\boldsymbol{x} \in \mathbb{R}^m$, the feature-weighted sum is:
\begin{equation}
\text{FS}(\boldsymbol{x}) = \sum_{i=1}^m w_i x_i,
\label{feature_sum}
\end{equation}
and the hyperplane is defined as:
\begin{equation}
y(\boldsymbol{x}) = \text{FS}(\boldsymbol{x}) - c = 0,
\label{hyperplane}
\end{equation}
where $c$ is a constant selected to optimize the split. Samples with $y(\boldsymbol{x}) < 0$ are assigned to the left subblock, and those with $y(\boldsymbol{x}) \geq 0$ to the right subblock.

For the $n$ samples in the block, denoted $\boldsymbol{x}_j \in \mathbb{R}^m$, compute $\text{FS}(\boldsymbol{x}_j) = \sum_{i=1}^m w_i x_{ij}$, forming the set $\mathcal{FS} = \{\text{FS}(\boldsymbol{x}_1), \dots, \text{FS}(\boldsymbol{x}_n)\}$, with subsets $\mathcal{FS}^\text{t}$ and $\mathcal{FS}^\text{nt}$ for target and non-target samples. Note that $\mathcal{FS}=\mathcal{FS}^\text{t}\cup\mathcal{FS}^{\text{nt}}$. 
Define:
\begin{align*}
\min \text{TFS} &= \min \{\mathcal{FS}^\text{t}\}, & \max \text{TFS} &= \max \{\mathcal{FS}^\text{t}\}, \\
\min \text{NFS} &= \min \{\mathcal{FS}^\text{nt}\}, & \max \text{NFS} &= \max \{\mathcal{FS}^\text{nt}\}.
\end{align*}
The central value $e = \frac{\min \text{NFS} + \max \text{NFS} + \min \text{TFS} + \max \text{TFS}}{4}$ serves as a fallback. Candidates for $c$ are $\min \text{NFS}$, $\max \text{NFS} + \delta_1$, $\min \text{TFS}$, $\max \text{TFS} + \delta_1$, and $e$, where $\delta_1 > 0$ is an infinitesimally small positive number used to ensure the purity of the subblock. A split where $c \neq e$ ensures that at least one of the resulting subblocks is a pure leaf block.  Further details are provided in Appendix~\ref{case}.

The selection of $c$ maximizes the number of pure samples in one subblock:
\begin{align*}
N_1 &= \left| \{ j \mid \text{FS}(\boldsymbol{x}_j) \in \mathcal{FS}^\text{t}, \text{FS}(\boldsymbol{x}_j) < \min \text{NFS} \} \right|, \\
N_2 &= \left| \{ j \mid \text{FS}(\boldsymbol{x}_j) \in \mathcal{FS}^\text{t}, \text{FS}(\boldsymbol{x}_j) > \max \text{NFS} \} \right|, \\
N_3 &= \left| \{ j \mid \text{FS}(\boldsymbol{x}_j) \in \mathcal{FS}^\text{nt}, \text{FS}(\boldsymbol{x}_j) < \min \text{TFS} \} \right|, \\
N_4 &= \left| \{ j \mid \text{FS}(\boldsymbol{x}_j) \in \mathcal{FS}^\text{nt}, \text{FS}(\boldsymbol{x}_j) > \max \text{TFS} \} \right|,
\end{align*}
with $N_{\max} = \max\{N_1, N_2, N_3, N_4\}$. $|\cdot|$ denotes the number of elements in a set. Then:
\begin{equation}
c = \begin{cases}
\min \text{NFS}, & \text{if } N_1 = N_{\max} \text{ and } N_{\max} \geq \gamma, \\
\max \text{NFS} + \delta_1, & \text{if } N_2 = N_{\max} \text{ and } N_{\max} \geq \gamma, \\
\min \text{TFS}, & \text{if } N_3 = N_{\max} \text{ and } N_{\max} \geq \gamma, \\
\max \text{TFS} + \delta_1, & \text{if } N_4 = N_{\max} \text{ and } N_{\max} \geq \gamma, \\
e, & \text{if } N_{\max} < \gamma,
\end{cases}
\label{c}
\end{equation}
where $\gamma > 0$ serves as the minimum sample threshold for generating pure leaf blocks. Its primary function is to prevent the creation of leaf blocks with too few samples, thereby ensuring that each leaf block resulting from a split contains a sufficient number of samples. If $N_{\max} < \gamma$, $c = e$ ensures both subblocks contain a certain number of samples. 
The construction of LHT follows Algorithm~\ref{alg:lht_build_appendix} provided in Appendix~\ref{alg_app}.
Figure \ref{LHT_sample} presents an example of sample allocation when $c=\min \text{TFS}$.
Additional cases where $N_{\max} \geq \gamma$ can be found in Appendix~\ref{case}.
Branching block 0 contains both target class and non-target class samples. The feature-weighted sum for each sample is computed to obtain the feature set $\mathcal{FS}$, as well as four specific feature-weighted sums: $\min \text{TFS}$, $\max \text{TFS}$, $\min \text{NFS}$, and $\max \text{NFS}$. 
\textcolor{black}{
Next, under the condition that $N_{\max}\geq\gamma$, the value of $c$ is selected from these four candidates to maximize the number of pure samples in the resulting leaf blocks. Then, check the value of $y(\boldsymbol{x})$, assigning samples with $y(\boldsymbol{x})<0$ to the left subblock and the remaining samples to the right subblock. If a subblock contains \textcolor{black}{too few samples} or samples from only one class, it is marked as a leaf block.
It is important to note that since the samples in block 2 are the remaining samples of block 0 after excluding those in block 1, $\mathbb{E}[X_i]^{\text{t}}$, $\mathbb{E}[X_i]^{\text{nt}}$, $\text{SD}_i$ and $w_i$ in block 2 change.} Consequently, the feature-weighted sum for the same sample may vary on different blocks. 
\begin{figure}[ht]
\vskip -0.1in
\begin{center}
\centerline{\includegraphics[width=0.45\textwidth]{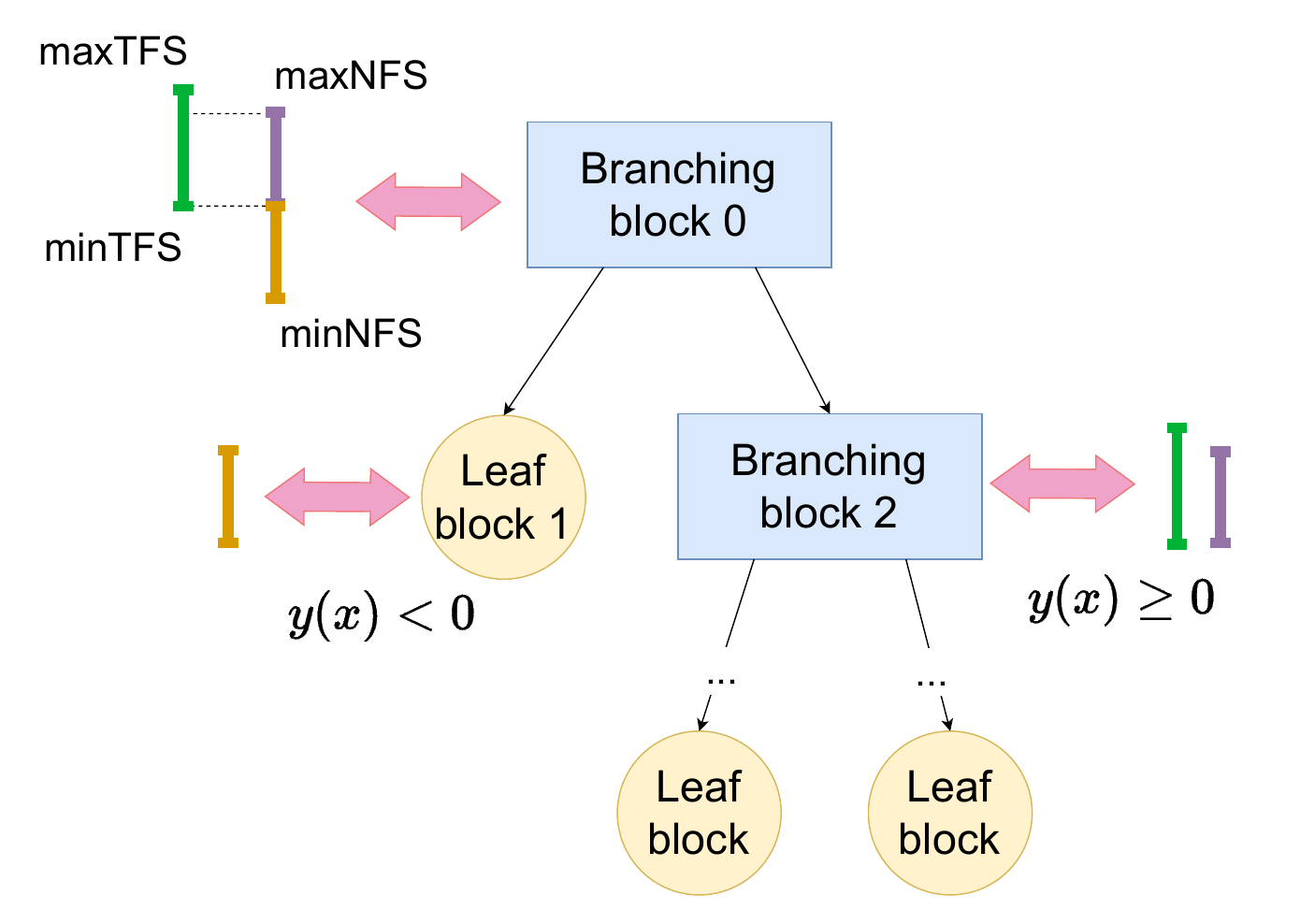}}
\caption{The case where $c=\min \text{TFS}$ is illustrated when $N_3 = N_{\max}$ and $N_{max}\geq\gamma$, where $c = \min \text{TFS}$. Samples with a feature-weighted sum smaller than $c$ are assigned to the left subblock 1, and the remaining samples are assigned to the right subblock 2. Since all samples in left subblock 1 are non-target class samples, it is marked as a leaf block. Right subblock 2 still contains mixed samples, and the allocation process continues based on the data within the block until all samples are properly classified.}
\label{LHT_sample}
\end{center}
\vskip -0.2in
\end{figure}

\subsection{Analysis of the Splitting Process} 
Having detailed the splitting mechanism based on the calculated threshold \(c\), we now establish a key property ensuring the process is well-defined. Specifically, we show that under standard conditions, the split always results in non-empty subblocks, guaranteeing the progression of the tree construction. To formalize this, we introduce the following assumptions:

\begin{assumption} \label{ass:non_trivial_block}
The block to be split contains at least two samples (\(n \geq 2\)). Furthermore, the block contains samples from both the target and non-target classes.
\end{assumption}

\begin{assumption} \label{ass:splittable}
The feature-weighted sums \(\text{FS}(\boldsymbol{x}_j)\) are not all identical within the block, ensuring that a split is feasible.
\end{assumption}

\begin{assumption} \label{ass:threshold_gamma}
The threshold \(\gamma\) used in the splitting criterion (defined in the selection process for \(c\)) is a positive integer satisfying \(1 \leq \gamma \leq n\).
\end{assumption}

Assumption \ref{ass:non_trivial_block} states that we only consider splitting non-pure blocks with sufficient data. Assumption \ref{ass:splittable} ensures the feature-weighted sums provide a basis for separation. Assumption \ref{ass:threshold_gamma} is a basic requirement for the hyperparameter \(\gamma\). Under these mild conditions, the proposed splitting procedure is guaranteed to be effective:

\begin{theorem}[Splitting Guarantee] \label{thm:convergence}
Under Assumptions \ref{ass:non_trivial_block}, \ref{ass:splittable}, and \ref{ass:threshold_gamma}, the block splitting process using the threshold \(c\) (selected as described in Section \ref{block2}) always partitions the block into two strictly non-empty subblocks: \(B_1 = \{\boldsymbol{x}_j \mid \text{FS}(\boldsymbol{x}_j) < c\}\) and \(B_2 = \{\boldsymbol{x}_j \mid \text{FS}(\boldsymbol{x}_j) \geq c\}\).
\end{theorem}

The proof is provided in Appendix~\ref{sec:appendix_proof}. 
This result confirms that the LHT splitting mechanism reliably partitions the data at each step, preventing the creation of empty nodes. \textcolor{black}{This non-empty partitioning is crucial as it guarantees the tree construction process \emph{terminates in a finite number of steps}}, ensuring the tree can be fully constructed.

\textbf{Computational Complexity.} Beyond ensuring valid splits, we analyze the computational cost of constructing the LHT. The following theorem characterizes the time complexity:

\begin{theorem}[Time Complexity] \label{thm:complexity}
Given a dataset with \(n\) samples and \(m\) features, the time complexity required to build an LHT up to a maximum depth \(d\) is \(O(mnd)\).
\end{theorem}

The detailed proof, which analyzes the cost of feature weighting, hyperplane determination, and data partitioning at each node and aggregates it across the tree levels, is provided in Appendix~\ref{sec:appendix_complexity_proof}. 
This complexity is comparable to standard axis-aligned decision tree algorithms (like CART when considering numeric features), demonstrating the practical efficiency of the LHT construction process despite using more expressive oblique splits.

\subsection{Membership Function}

\textcolor{black}{The membership function of an LHT is \textcolor{black}{a piecewise linear function derived via local least squares fitting within each leaf block.}} The classification task can be completed using the membership function of the leaf block corresponding to the input data.

\textcolor{black}{For a given class, the LHT assigns labels $p_i = 1$ to target samples and $p_i = 0$ to non-target samples, where $p_i \in \{0, 1\}$ for $i = 1, 2, \dots, n$. Within each leaf block, a linear function $\hat{p}(\boldsymbol{x})$ is fitted to these labels using least squares regression. To represent the degree of membership in [0,1], the membership function is:
\begin{equation}
\mu(\boldsymbol{x}) = \max\left\{0, \min\left\{\hat{p}(\boldsymbol{x}), 1\right\}\right\}.
\end{equation}
See Appendix~\ref{ap1} for derivation details. 
The procedure for traversing the tree to a leaf node and computing the membership value is outlined in Algorithm \ref{alg:lht_predict_appendix} (Appendix \ref{alg_app}). Since the membership function has a computational complexity of \(O(mn)\), it does not affect the overall complexity of constructing the LHT, which remains \(O(mnd)\). See Appendix~\ref{note} for details.} 

\subsection{ Universal Approximation Capability}

We prove that LHTs can universally approximate continuous functions on a compact set \( K \subset \mathbb{R}^m \).

\begin{theorem}
Let \( K \subset \mathbb{R}^m \) be a compact set and \( g: K \to [0,1] \) be a continuous function. Then, for any \( \epsilon > 0 \), there exists an LHT-defined function \( f_{\text{LHT}}: K \to [0,1] \) such that:
\[ \sup_{\mathbf{x} \in K} |f_{\text{LHT}}(\mathbf{x}) - g(\mathbf{x})| < \epsilon \]
\end{theorem}

See Appendix~\ref{sec:universal_approximation} for proof.

\subsection{LH Forest}
\label{LHf}

When the performance of a single LHT is limited, an LH Forest, comprising multiple LHTs, can be constructed to enhance feature extraction and improve classification accuracy. Each class is associated with several LHTs, and their outputs are aggregated to produce the final classification decision.

For datasets with many samples but few features, a random forest-inspired approach is adopted: multiple LHTs are trained on random subsets of the data to promote diversity and mitigate overfitting. Conversely, for datasets with fewer samples but high-dimensional features, distinct feature subsets are selected to train individual LHTs, capturing diverse data aspects.

The feature weights \( w_i \) in each LHT quantify the discriminative power of individual features, guiding effective feature selection. To build an LH Forest with \( t \) trees per class, a feature selection strategy adjusts the threshold \( \beta \) for the \( i \)-th tree (\( i = 0, 1, \dots, t-1 \)) as \( \beta_i = \beta' \cdot \frac{i}{t} \), where \( \beta' \in [0, 1] \) is a predefined constant controlling selection strictness. This approach ensures diversity across trees while prioritizing features with high discriminative ability.

\subsection{Interpretability Mechanisms}

The LHT separates target and non-target samples through recursive hyperplane splits, with leaf nodes employing piecewise linear membership functions derived from least squares fitting. This hierarchical structure and transparent design render LHT inherently interpretable.

Each branching block corresponds to a hyperplane defined by feature weights \( w_i \), where \( |w_i| \) measures the importance of the \( i \)-th feature in distinguishing between classes at that split. By traversing the tree’s decision paths, one can quantify each feature’s contribution at every branching block to the classification outcome. The tree’s transparency facilitates visualization and analysis of feature interactions across splits.
To illustrate this interpretability, we evaluate feature contributions in an LHT trained on the Wine dataset. Results are detailed in Appendix~\ref{ap2}.

\section{Experiments}
\label{exp}

\subsection{Comparison with Oblique Trees and CART}
\label{sec:comparison} 

In this section, we evaluate the performance of our proposed LHT model against several established benchmarks. These include SOTA oblique decision tree algorithms: Tree Alternating Optimization (TAO)~\cite{NEURIPS2018_185c29dc}, Dense Gradient Trees (DGT)~\cite{karthikeyanlearning}, DTSemNet~\cite{panda2024vanilla}, and the classic CART algorithm~\cite{breiman1984cart}. 

\begin{itemize}
    \item \textbf{TAO}~\cite{NEURIPS2018_185c29dc} employs alternating optimization at the tree-depth level to minimize misclassification errors, capable of handling various tree types and learning sparse oblique trees via sparsity penalties.
    \item \textbf{DGT}~\cite{karthikeyanlearning} utilizes end-to-end gradient descent, leveraging techniques like over-parameterization and straight-through estimators, to train oblique trees effectively for both standard supervised and online learning scenarios.
    \item \textbf{DTSemNet}~\cite{panda2024vanilla} encodes oblique decision trees as semantically equivalent neural networks using ReLU activations and linear operations, enabling optimization via standard gradient descent without relying on approximation techniques.
    \item \textbf{CART}~\cite{breiman1984cart} is a foundational algorithm that recursively partitions data, serving as a standard baseline for classification and regression trees.
\end{itemize}

We conducted extensive comparative experiments across several diverse datasets. The accuracy results are presented in Table~\ref{Oblique}. The results demonstrate that LHT consistently achieves superior or competitive performance, outperforming TAO, DGT, DTSemNet, and CART on a significant majority of the evaluated datasets.

Notably, LHT exhibits a distinct advantage on several challenging datasets, such as those with high dimensionality or large scale (e.g., MNIST, Letter, Avila Bible, Optical Recog.). This suggests LHT's effectiveness in learning complex decision boundaries. While baseline methods achieve performance comparable to LHT on some simpler datasets (e.g., Acute Inflam., Banknote), LHT's overall strong performance across the spectrum of tasks underscores its effectiveness and potential as a SOTA oblique decision tree algorithm.

\begin{table*}[ht]
	\caption{Percentage accuracy on classification tasks with CART and SOTA Oblique Decision Trees for the datasets reported in \cite{NEURIPS2018_185c29dc} and \cite{panda2024vanilla}. $(n, m, k)$ denote samples, features, and classes. Averaged accuracy ± std is reported over 10 runs.The results for CART, TAO, DGT and DTSemNet are copied from \cite{panda2024vanilla}. Best results per row are in \textbf{bold}.}
	\label{Oblique}
	\centering
	\begin{center}
				\begin{tabular}{{@{}l@{\hspace{2pt}}c@{\hspace{8pt}}c@{\hspace{8pt}}c@{\hspace{8pt}}c@{\hspace{3pt}}c@{\hspace{2pt}}c@{}}}
					\toprule
					\multirow{2}{*}{Dataset} &\multirow{2}{*}{$(n,m,k)$} & \multicolumn{5}{c}{Percentage Accuracy (\( \uparrow \))}\\
                    \cmidrule(lr){3-7}
					 & &CART \cite{breiman1984cart} & TAO \cite{NEURIPS2018_185c29dc} & DGT \cite{karthikeyanlearning} & DTSemNet \cite{panda2024vanilla} & LHT   \\
\midrule
Protein &  (24387,357,3)   & 57.5 $\pm$ 0.0 & 68.4 $\pm$ 0.2 & 67.8 $\pm$ 0.4 & 68.6 $\pm$ 0.2 & \bf{69.1 $\pm$ 0.8} \\
SatImages &  (6435,36,6) & 84.1 $\pm$ 0.3 & 87.4 $\pm$ 0.3 & 86.6 $\pm$ 0.9 & 87.5 $\pm$ 0.5 & \bf{90.9 $\pm$ 0.6} \\
Segment &   (2310,19,7)   & 94.2 $\pm$ 0.8 & 95.0 $\pm$ 0.8 & 95.8 $\pm$ 1.1 & 96.1 $\pm$ 0.5 & \bf{97.0 $\pm$ 0.7} \\
Pendigits &   (10992,16,10)   & 89.9 $\pm$ 0.3 & 96.0 $\pm$ 0.3 & 96.3 $\pm$ 0.2 & 97.0 $\pm$ 0.3 & \bf{97.1 $\pm$ 0.3} \\
Connect-4 &   (67757,126,3)    & 74.0 $\pm$ 0.6 & 81.2 $\pm$ 0.2 & 79.5 $\pm$ 0.2 & 82.0 $\pm$ 0.3 & \bf{82.1 $\pm$ 0.2} \\
MNIST&  (70000,784,10) & 85.5 $\pm$ 0.1 & 95.0 $\pm$ 0.1 & 94.0 $\pm$ 0.3 & 96.1 $\pm$ 0.1 & \bf{97.4 $\pm$ 0.3} \\
SensIT &   (98528,100,3)   & 78.3 $\pm$ 0.0 & 82.5 $\pm$ 0.1 & 83.6 $\pm$ 0.2 & 84.2 $\pm$ 0.1 & \bf{86.6 $\pm$ 0.1} \\
Letter & (20000,16,26)  & 70.1 $\pm$ 0.1 & 87.4 $\pm$ 0.4 & 86.1 $\pm$ 0.7 & 89.1 $\pm$ 0.2 & \bf{95.0 $\pm$ 0.2} \\
Balance Scale &  (625,4,3)   & 74.9 $\pm$ 3.6 & 77.4 $\pm$ 3.1 & 88.6 $\pm$ 1.7 & 90.2 $\pm$ 2.2 & \bf{90.3 $\pm$ 0.7} \\
Banknote &  (1372,4,2)    & 93.6 $\pm$ 2.2 & 96.6 $\pm$ 1.3 & \bf{99.8 $\pm$ 0.4} & \bf{99.8 $\pm$ 0.4} & \bf{99.8 $\pm$ 0.1} \\
Blood Trans. &  (784,4,2)  & 77.1$\pm$1.8 & 76.9 $\pm$ 2.1 & 78.3 $\pm$ 2.4 & 78.5 $\pm$ 1.7 & \bf{79.5 $\pm$ 2.4} \\
Acute Inflam. 1 &  (120,6,2)    & \bf{100 $\pm$ 0} & 99.7 $\pm$ 1.2 & \bf{100 $\pm$ 0} & \bf{100 $\pm$ 0} & \bf{100 $\pm$ 0} \\
Acute Inflam. 2 &   (120,6,2)    & 99.0 $\pm$ 2.6 & 99.0 $\pm$ 2.6 & \bf{100 $\pm$ 0} & \bf{100 $\pm$ 0} & \bf{100 $\pm$ 0} \\
Car Evaluation &  (1728,6,4)   & 84.3 $\pm$ 1.4 & 84.5 $\pm$ 1.5 & 92.1 $\pm$ 2.4 & 93.3 $\pm$ 2.2 & \bf{97.1 $\pm$ 1.0} \\
Breast Cancer &   (699,9,2)  & 94.7 $\pm$ 1.7 & 94.7 $\pm$ 1.7 & 97.2 $\pm$ 1.0 & 97.2 $\pm$ 1.3 & \bf{97.3 $\pm$ 1.0} \\
Avila Bible &  (20867,10,12) & 54.0 $\pm$ 1.3 & 55.8 $\pm$ 0.8 & 59.7 $\pm$ 1.8 & 62.2 $\pm$ 1.4 & \bf{73.2 $\pm$ 0.3} \\
Wine-Red &  (1599,11,6)  & 55.9 $\pm$ 2.3 & 56.9 $\pm$ 2.5 & 56.6 $\pm$ 1.4 & 58.6 $\pm$ 2.2 & \bf{63.1 $\pm$ 0.9} \\
Wine-White &  (6898,11,7) & 52.0 $\pm$ 1.3 & 52.3 $\pm$ 1.5 & 52.1 $\pm$ 1.6 & 53.5 $\pm$ 1.4 & \bf{61.7 $\pm$ 0.9} \\
Dry Bean &  (13611,16,7)  & 80.5 $\pm$ 1.9 & 83.2 $\pm$ 1.5 & 89.0 $\pm$ 1.6 & 91.4 $\pm$ 0.5 & \bf{92.5 $\pm$ 0.3} \\
Climate Crashes &  (540,18,2)   & 91.8 $\pm$ 1.8 & 90.6 $\pm$ 2.2 & 92.4 $\pm$ 2.6 & 92.9 $\pm$ 1.4 & \bf{93.1 $\pm$ 2.4} \\
Conn. Sonar & (208,60,2)   & 70.6 $\pm$ 6.6 & 70.9 $\pm$ 5.8 & 80.8 $\pm$ 4.5 & 82.1 $\pm$ 5.1 & \bf{82.5 $\pm$ 5.8} \\
Optical Recog. & (5620,64,10)   & 53.2 $\pm$ 3.2 & 64.6 $\pm$ 6.5 & 91.9 $\pm$ 1.0 & 93.3 $\pm$ 1.0 & \bf{95.5 $\pm$ 0.6} \\
					\bottomrule
				\end{tabular}
	\end{center}
    \vskip -0.12in
\end{table*}

\subsection{Comparison with Representative Methods}

To further evaluate the performance of LHT, we test it on several public datasets from \cite{NEURIPS2018_185c29dc}, covering problems with small, medium, and large sample sizes, and \textcolor{black}{compare its results against a suite of highly competitive, SOTA tree-based ensemble methods. These methods, widely recognized for their effectiveness across diverse tabular data tasks, are extensively adopted in both academic research and industrial applications.} The comparative methods are described as follows:

    \begin{itemize}
        \item \textbf{RF}~\cite{breiman2001random}: The Random Forest classifier implemented in \texttt{scikit-learn}~\cite{pedregosa2011scikit}.
        \item \textbf{XGBoost}~\cite{chen2016XGBoost}: An efficient gradient boosting decision tree implementation, using the official \texttt{XGBoost} library\footnote{\url{https://XGBoost.ai/}}.
        \item \textbf{CatBoost}~\cite{prokhorenkova2018CatBoost}: A gradient boosting library effective with categorical features, using the official \texttt{CatBoost} library\footnote{\url{https://CatBoost.ai/}}.
        \item \textbf{LightGBM}~\cite{ke2017LightGBM}: A fast, low-memory gradient boosting framework, using the official \texttt{LightGBM} library\footnote{\url{https://LightGBM.readthedocs.io/}}.
    \end{itemize}

\begin{table*}[t]
	\caption{Percentage accuracy comparison with representative methods. $(n, m, k)$ denote samples, features, and classes. Mean accuracy ± std over 10 runs. Best results per row are in \textbf{bold}.  The methods include Random Forest (RF), XGBoost (XGB), CatBoost (CatB), and LightGBM (LGBM).}
	\label{sample-tables}
	\centering
	\begin{center}
				\begin{tabular}{@{}l@{\hspace{5pt}}c@{\hspace{5pt}}c@{\hspace{5pt}}c@{\hspace{5pt}}c@{\hspace{5pt}}c@{\hspace{5pt}}c@{}}
					\toprule
					\multirow{2}{*}{Dataset} & \multirow{2}{*}{($n$,$m$,$k$)} & \multicolumn{5}{c}{Percentage Accuracy (\( \uparrow \))}\\
                    \cmidrule(lr){3-7}
                    
					 & &RF \cite{breiman2001random} & XGB \cite{chen2016XGBoost} & CatB \cite{prokhorenkova2018CatBoost} & LGBM \cite{ke2017LightGBM} & LHT \\
					\midrule
Protein &  (24387,357,3)   & 48.2 $\pm$ 0.6 &  48.8 $\pm$ 0.6  &  48.8 $\pm$ 0.6  &  48.3 $\pm$ 0.6 & \bf{69.1 $\pm$ 0.8} \\
SatImages &  (6435,36,6) &  90.2 $\pm$ 0.7  & 90.2 $\pm$ 0.7 &  90.8 $\pm$ 0.7  &  90.8$\pm$0.6  & \bf{90.9 $\pm$ 0.6} \\
Segment &   (2310,19,7)   & 96.7 $\pm$ 0.8 &  \bf{97.8 $\pm$ 0.7} &  96.7 $\pm$ 0.9  & 97.4 $\pm$ 0.7 & 97.0 $\pm$ 0.7 \\
Pendigits &   (10992,16,10)   &  96.3 $\pm$ 0.3  &  96.4 $\pm$ 0.3  &  \bf{97.1 $\pm$ 0.3}  &  96.4 $\pm$ 0.3  & \bf{97.1 $\pm$ 0.3} \\
Connect-4 &   (67757,126,3)    &  78.5 $\pm$ 0.4  &  85.8 $\pm$ 0.3 &  84.1 $\pm$ 0.3  &  \bf{87.7} $\pm$ 0.3  &  82.1 $\pm$ 0.2  \\
MNIST&  (70000,784,10) & 96.6 $\pm$ 0.2 & 97.5 $\pm$ 0.2 &  96.9 $\pm$ 0.2  &  \bf{98.0 $\pm$ 0.1} &  97.4 $\pm$ 0.3  \\
SensIT &   (98528,100,3)   & 85.9 $\pm$ 0.2& \bf{87.5 $\pm$ 0.2} & 87.3 $\pm$ 0.2 &  87.3 $\pm$ 0.3  & 86.6 $\pm$ 0.1 \\
Letter & (20000,16,26)  &  94.9 $\pm$ 0.3  &  95.9 $\pm$ 0.3  & 96.1 $\pm$ 0.3 & \bf{96.7 $\pm$ 0.3} & 95.0 $\pm$ 0.2 \\
					Wine & (178,13,3) & 97.5 $\pm$ 1.1 & 96.8 $\pm$ 1.1 & 97.0 $\pm$ 0.9 & 97.3 $\pm$ 0.8 & \bf{97.8 $\pm$ 1.6} \\
					Seeds & (210,7,3) & 91.1 $\pm$ 1.3 & 93.4 $\pm$ 1.4 & 93.1 $\pm$ 1.6 & 92.6 $\pm$ 1.5 & \bf{94.5 $\pm$ 2.1} \\
					WDBC & (569,30,2) & 95.7 $\pm$ 1.3 & 96.6 $\pm$ 0.7 & 96.6 $\pm$ 1.3 & 96.1 $\pm$ 0.6 & \bf{98.2 $\pm$ 1.1} \\
					Banknote & (1372,3,2) & 98.9 $\pm$ 0.2 & 98.9 $\pm$ 0.2 & 99.5 $\pm$ 0.2 & 99.2 $\pm$ 0.2 & \bf{99.8 $\pm$ 0.2} \\
                    	Rice & (3810,7,2) & 92.2 $\pm$ 0.3 & 92.2 $\pm$ 0.4 & \bf{92.7 $\pm$ 0.4} & 92.2 $\pm$ 0.3 & 91.3 $\pm$ 0.5 \\
					Spambase & (4601,57,2) & 93.2 $\pm$ 0.3 & \bf{94.2 $\pm$ 0.3} & 94.1 $\pm$ 0.3 & \bf{94.2 $\pm$ 0.3} & 93.8 $\pm$ 0.6 \\
					EEG & (14980,14,2) & 91.1 $\pm$ 0.2 & 94.6 $\pm$ 0.3 & 92.6 $\pm$ 0.4 & 92.4 $\pm$ 0.3 & \bf{94.8 $\pm$ 0.3} \\
					MAGIC  & (19020,10,2) & 84.4 $\pm$ 0.2 & 87.8 $\pm$ 0.2 & 87.5 $\pm$ 0.2 & \bf{87.8 $\pm$ 0.2} & 86.1 $\pm$ 0.3 \\
					SKIN & (245057,3,2) & 99.84 $\pm$ 0.01 & 99.93 $\pm$ 0.01 & 99.89 $\pm$ 0.01 & 99.92 $\pm$ 0.01 & \bf{99.97 $\pm$ 0.01} \\
					\bottomrule
				\end{tabular}
	\end{center}
    \vskip -0.1in
\end{table*}

Table~\ref{sample-tables} shows the comparison results between LHT and representative ensemble methods, including RF\ and gradient boosting variants (XGBoost, CatBoost, LightGBM), which represent SOTA performance on many tabular tasks.  \textcolor{black}{Despite this highly challenging competition, the results indicate that LHT exhibits remarkable competitiveness against these powerful baselines.} Notably, LHT achieves the highest accuracy on a large number of datasets such as Protein, SatImages, Pendigits, Wine, Seeds, WDBC, Banknote, EEG, and SKIN. On the remaining datasets where ensemble methods, particularly LightGBM, XGBoost, or CatBoost, achieve the top performance (e.g., Connect-4, MNIST, Letter, Rice), LHT generally demonstrates comparable or closely competitive results. Overall, this rigorous comparison validates the effectiveness of LHT, demonstrating its capability to effectively challenge and often outperform widely-adopted, sophisticated ensemble techniques across diverse benchmarks.

\textbf{Computational Efficiency.}
To evaluate the training efficiency of LHT, we compare its training time with RF\ and gradient boosting methods on three representative datasets in Table~\ref{time-tables}. The training of LHT involves sequential learning of trees; however, the learning process for each tree can theoretically be performed independently. The values in parentheses indicate the maximum time required to learn a single tree among all trees, representing the potential shortest training time for LHT under ideal parallelization conditions.
The results show that the total training time of LHT generally falls between that of RF\ and  CatBoost. For instance, on the Pendigits dataset, LHT's total time (2.5454s) is slower than RF\ (1.2444s) but significantly faster than CatBoost\ (90.2304s), and is comparable to XGBoost\ (2.4511s) and LightGBM\ (1.0643s). Importantly, the maximum single-tree training time for LHT (values in parentheses) is very short. This suggests that with parallel computation, the training process of LHT could be substantially accelerated, potentially becoming faster than current gradient boosting methods. This offers a potential advantage for LHT in scenarios requiring rapid model training.

\textbf{Implementation Details.} The implementation of LHT is available at link\footnote{\url{https://github.com/Hongyi-Li-sz/LHT_model}}. Our method builds upon a GitHub repository for MNIST classification\footnote{\url{https://github.com/Bill-Armstrong/Real-Time-Machine-Learning}}.
All experiments were conducted on an AMD Ryzen 7 5800H processor with 16GB RAM, using a single CPU core. For pre-partitioned datasets, we used the provided training and test sets. For non-pre-partitioned datasets, we applied a random 80\%/20\% train-test split. Experimental hyperparameters and dataset information are provided in Appendix~\ref{hyperpara}.

\begin{table*}[t]
\vskip -0.2in
	\caption{Training time (seconds) comparison on selected classification tasks. Values for LHT show total sequential time, with maximum single-tree training time (potential parallel time) in parentheses. }
	\label{time-tables}
	\centering
	\begin{center}
				\begin{tabular}{@{}l@{\hspace{8pt}}c@{\hspace{10pt}}c@{\hspace{10pt}}c@{\hspace{10pt}}c@{\hspace{10pt}}c@{\hspace{10pt}}c@{}}
					\toprule
					\multirow{2}{*}{Dataset} & \multirow{2}{*}{($n$,$m$,$k$)} & \multicolumn{5}{c}{Training Time (\( \downarrow \))}\\
                    \cmidrule(lr){3-7}
                    
					 & &RF \cite{breiman2001random} & XGB \cite{chen2016XGBoost} & CatB \cite{prokhorenkova2018CatBoost} & LGBM \cite{ke2017LightGBM} & LHT \\
					\midrule

SatImages &  (6435,36,6) &  0.3750  & 0.6689 &  23.5407  &  0.5876  & 2.5115 (0.0291) \\
Segment &   (2310,19,7)   & 0.2816 &  0.5869 &  4.1902  & 0.1424 & 0.0312 (0.0092) \\
Pendigits &   (10992,16,10)    &  1.2444  &  2.4511 & 90.2304  &  1.0643  &    2.5454 (0.0272)\\

					\bottomrule
				\end{tabular}
	\end{center}
    \vskip -0.12in
\end{table*}

\section{Related Work}
\label{sec:related_work}

Our work builds upon and distinguishes itself from extensive research in tree-based models, particularly standard decision trees and their oblique counterparts.

\textbf{Decision Trees (Axis-Parallel).}
Foundational decision tree algorithms such as ID3 \cite{quinlan1986induction}, C4.5 \cite{quinlan1993c4}, and CART \cite{breiman1984cart} established core induction principles using criteria like information gain, gain ratio, and Gini impurity. These methods construct trees with \emph{axis-parallel} splits, partitioning the feature space along individual feature dimensions. 
While interpretable, single decision trees can be unstable and prone to overfitting. Ensemble techniques were developed to mitigate these issues. Bagging methods, most notably Random Forests \cite{breiman2001random}, average predictions from multiple trees grown on bootstrapped data subsets, reducing variance \cite{breiman1996bagging}. Ongoing research explores improvements like dataset reconstruction from forests \cite{pmlr-v235-ferry24a} and density estimation variants \cite{pmlr-v162-wen22c}. Boosting methods, in contrast, build models sequentially, with each new tree correcting errors of the ensemble \cite{breiman1996bias}. AdaBoost \cite{freund1997decision} adaptively weights samples, while Gradient Boosting Decision Trees (GBDT) \cite{friedman2001greedy} employ gradient descent in function space. Modern optimized GBDT frameworks like XGBoost \cite{chen2016XGBoost}, LightGBM \cite{ke2017LightGBM}, and CatBoost \cite{prokhorenkova2018CatBoost} achieve SOTA results on many tabular datasets, primarily relying on ensembles of axis-parallel trees.

\textbf{Oblique Decision Trees.}
To provide a more flexible segmentation, \emph{oblique decision trees} utilize splits based on hyperplanes (\(\boldsymbol{w}^\top \boldsymbol{x} - c = 0\)), allowing them to capture linear relationships between features more directly. The primary challenge lies in determining the hyperplane parameters (\(\boldsymbol{w}, c\)). A variety of approaches have been proposed. Early work like OC1 \cite{murthy1993oc1}, specifically designed for oblique trees, pioneered heuristic search methods, employing techniques such as random search and coordinate descent to find effective oblique splits. Many subsequent methods also involve \textit{iterative optimization or heuristics}. GUIDE \cite{loh2014fifty} grows a tree greedily and recursively and prioritizes unbiased feature selection. \textit{Global optimization} approaches aim to find the optimal tree structure. 
\textcolor{black}{Learning optimal trees is computationally challenging, as even the simplest case (binary inputs and outputs) is NP-hard \cite{laurent1976constructing,hancock1996lower}.}
Notably, Optimal Classification Trees (OCT) \cite{bertsimas2017optimal} leverage mathematical optimization by formulating the problem as a mixed-integer program (MIP) to find a globally optimal tree. However, solving the inherent NP-hard MIP problem limits its application typically to smaller datasets \cite{zharmagambetov2021non}. Recent studies continue to explore different facets of oblique trees. 
\textcolor{black}{CO2 \cite{NIPS2015_1579779b,norouzi2015co2} formulates a convex-concave upper bound on the tree’s empirical loss, optimizing it via stochastic gradient descent (SGD) given an initial tree structure (e.g., from CART). While SGD enables scalability to large datasets, each iteration may not reduce the objective. }
TAO~\cite{NEURIPS2018_185c29dc,pmlr-v119-zharmagambetov20a} uses alternating optimization for a non-greedy global search and can be extended to generate oblique splits. DGT~\cite{karthikeyanlearning} utilizes end-to-end gradient descent, leveraging techniques like over-parameterization and straight-through estimators, to train oblique trees effectively for both standard supervised and online learning scenarios.
DTSemNet~\cite{panda2024vanilla} encodes oblique decision trees as semantically equivalent neural networks using ReLU activations and linear operations, enabling optimization via standard gradient descent without relying on approximation techniques.

Our proposed LHT contributes to the oblique tree literature by introducing a distinct hyperplane construction method. Unlike the aforementioned techniques that predominantly rely on iterative optimization or complex global formulations, LHT employs a \emph{non-iterative, statistically-driven} procedure. At each node, the hyperplane orientation \(\boldsymbol{w}\) is \textcolor{black}{\emph{directly}} derived from the difference in feature expectations between the target and non-target classes. To our knowledge, this specific mechanism—using feature expectation differences to directly guide oblique hyperplane construction without iteration—is unique to LHT, offering a novel paradigm for building oblique decision trees.

\section{Limitation}
\label{lim}
\textcolor{black}{Similar to other tree-based models, LHT may be less effective than deep learning approaches when dealing with data that has strong spatial structure and homogeneity, such as images, as deep learning models are better suited for capturing complex spatial patterns. As a result, LHT's strengths are most evident in areas like tabular data analysis. Future work will investigate the extension of LHT to image classification problems, aiming to achieve competitive or high-accuracy performance.}

\section{Discussion and Conclusion}
This paper presents a novel tree-based model, LHT, which uses feature expectation differences to directly construct oblique hyperplanes, eliminating the need for iterative optimization.  We evaluate the performance of LHT on several public datasets. The experimental results show that LHT exhibits remarkable performance in processing tabular data. Moreover, the contribution of each feature at every branching step is clearly visible. This transparency in decision-making is particularly valuable within the current context of responsible AI, where demands for fairness, transparency, and security are paramount. As such, LHT has the potential to play an increasingly significant role in addressing these challenges.

\bibliography{example_paper}

\begin{thebibliography}{10}

\bibitem{chen2016XGBoost}
Tianqi Chen and Carlos Guestrin.
\newblock Xgboost: A scalable tree boosting system.
\newblock In {\em Proceedings of the 22nd ACM Sigkdd International Conference
  on Knowledge Discovery and Data Mining}, pages 785--794, 2016.

\bibitem{shwartz-ziv2021tabular}
Ravid Shwartz-Ziv and Amitai Armon.
\newblock Tabular data: Deep learning is not all you need.
\newblock In {\em 8th ICML Workshop on Automated Machine Learning (AutoML)},
  2021.

\bibitem{grinsztajn2022tree}
Leo Grinsztajn, Edouard Oyallon, and Gael Varoquaux.
\newblock Why do tree-based models still outperform deep learning on typical
  tabular data?
\newblock In {\em Advances in Neural Information Processing Systems},
  volume~35, pages 507--520. Curran Associates, Inc., 2022.

\bibitem{NEURIPS2021_285a25c1}
Mohammadreza Armandpour, Ali Sadeghian, and Mingyuan Zhou.
\newblock Convex polytope trees.
\newblock In {\em Advances in Neural Information Processing Systems},
  volume~34, pages 5038--5051. Curran Associates, Inc., 2021.

\bibitem{marton2024gradtree}
Sascha Marton, Stefan L{\"u}dtke, Christian Bartelt, and Heiner Stuckenschmidt.
\newblock Gradtree: Learning axis-aligned decision trees with gradient descent.
\newblock In {\em Proceedings of the AAAI}, volume~38, pages 14323--14331,
  2024.

\bibitem{zharmagambetov2021non}
Arman Zharmagambetov, Suryabhan~Singh Hada, Magzhan Gabidolla, and Miguel~A
  Carreira-Perpin{\'a}n.
\newblock Non-greedy algorithms for decision tree optimization: An experimental
  comparison.
\newblock In {\em 2021 International Joint Conference on Neural Networks
  (IJCNN)}, pages 1--8. IEEE, 2021.

\bibitem{NEURIPS2018_185c29dc}
Miguel~A. Carreira-Perpinan and Pooya Tavallali.
\newblock Alternating optimization of decision trees, with application to
  learning sparse oblique trees.
\newblock In {\em Advances in Neural Information Processing Systems},
  volume~31, pages 1219--1229. Curran Associates, Inc., 2018.

\bibitem{karthikeyanlearning}
Ajaykrishna Karthikeyan, Naman Jain, Nagarajan Natarajan, and Prateek Jain.
\newblock Learning accurate decision trees with bandit feedback via quantized
  gradient descent.
\newblock {\em Transactions on Machine Learning Research}, 2022.

\bibitem{panda2024vanilla}
Subrat~Prasad Panda, Blaise Genest, Arvind Easwaran, and Ponnuthurai~Nagaratnam
  Suganthan.
\newblock Vanilla gradient descent for oblique decision trees.
\newblock {\em arXiv preprint arXiv:2408.09135}, 2024.

\bibitem{breiman1984cart}
Leo Breiman, Jerome Friedman, Richard Olshen, and Charles Stone.
\newblock {\em Classification and Regression Trees}.
\newblock Chapman and Hall/CRC, 1984.

\bibitem{breiman2001random}
Leo Breiman.
\newblock Random forests.
\newblock {\em Machine Learning}, 45:5--32, 2001.

\bibitem{pedregosa2011scikit}
Fabian Pedregosa, Ga{\"e}l Varoquaux, Alexandre Gramfort, Vincent Michel,
  Bertrand Thirion, Olivier Grisel, Mathieu Blondel, Peter Prettenhofer, Ron
  Weiss, Vincent Dubourg, et~al.
\newblock Scikit-learn: Machine learning in {P}ython.
\newblock {\em Journal of Machine Learning Research}, 12:2825--2830, 2011.

\bibitem{prokhorenkova2018CatBoost}
Liudmila Prokhorenkova, Gleb Gusev, Aleksandr Vorobev, Anna~Veronika Dorogush,
  and Andrey Gulin.
\newblock Catboost: unbiased boosting with categorical features.
\newblock In {\em Advances in Neural Information Processing Systems},
  volume~31, pages 6639--6649, 2018.

\bibitem{ke2017LightGBM}
Guolin Ke, Qi~Meng, Thomas Finley, Taifeng Wang, Wei Chen, Weidong Ma, Qiwei
  Ye, and Tie-Yan Liu.
\newblock Lightgbm: A highly efficient gradient boosting decision tree.
\newblock In {\em Advances in Neural Information Processing Systems},
  volume~30, pages 3149--3157, 2017.

\bibitem{quinlan1986induction}
J.~Ross Quinlan.
\newblock Induction of decision trees.
\newblock {\em Machine Learning}, 1:81--106, 1986.

\bibitem{quinlan1993c4}
J~Ross Quinlan.
\newblock C4. 5: Programs for machine learning.
\newblock {\em Machine Learning}, 16:235--240, 1993.

\bibitem{breiman1996bagging}
Leo Breiman.
\newblock Bagging predictors.
\newblock {\em Machine Learning}, 24:123--140, 1996.

\bibitem{pmlr-v235-ferry24a}
Julien Ferry, Ricardo Fukasawa, Timoth\'{e}e Pascal, and Thibaut Vidal.
\newblock Trained random forests completely reveal your dataset.
\newblock In {\em Proceedings of the 41st International Conference on Machine
  Learning}, volume 235, pages 13545--13569. PMLR, 21--27 Jul 2024.

\bibitem{pmlr-v162-wen22c}
Hongwei Wen and Hanyuan Hang.
\newblock Random forest density estimation.
\newblock In {\em Proceedings of the 39th International Conference on Machine
  Learning}, volume 162, pages 23701--23722. PMLR, 17--23 Jul 2022.

\bibitem{breiman1996bias}
Leo Breiman.
\newblock Bias, variance, and arcing classifiers.
\newblock {\em Statistics}, 1996.

\bibitem{freund1997decision}
Yoav Freund and Robert~E Schapire.
\newblock A decision-theoretic generalization of on-line learning and an
  application to boosting.
\newblock {\em Journal of Computer and System Sciences}, 55(1):119--139, 1997.

\bibitem{friedman2001greedy}
Jerome~H Friedman.
\newblock Greedy function approximation: a gradient boosting machine.
\newblock {\em The Annals of Statistics}, pages 1189--1232, 2001.

\bibitem{murthy1993oc1}
Sreerama~K Murthy, Simon Kasif, Steven Salzberg, and Richard Beigel.
\newblock {OC1}: A randomized algorithm for building oblique decision trees.
\newblock In {\em Proceedings of the AAAI}, volume~93, pages 322--327.
  Citeseer, 1993.

\bibitem{loh2014fifty}
Wei-Yin Loh.
\newblock Fifty years of classification and regression trees.
\newblock {\em International Statistical Review}, 82(3):329--348, 2014.

\bibitem{laurent1976constructing}
Hyafil Laurent and Ronald~L Rivest.
\newblock Constructing optimal binary decision trees is {NP}-complete.
\newblock {\em Information Processing Letters}, 5(1):15--17, 1976.

\bibitem{hancock1996lower}
Thomas Hancock, Tao Jiang, Ming Li, and John Tromp.
\newblock Lower bounds on learning decision lists and trees.
\newblock {\em Information and Computation}, 126(2):114--122, 1996.

\bibitem{bertsimas2017optimal}
Dimitris Bertsimas and Jack Dunn.
\newblock Optimal classification trees.
\newblock {\em Machine Learning}, 106:1039--1082, 2017.

\bibitem{NIPS2015_1579779b}
Mohammad Norouzi, Maxwell Collins, Matthew~A Johnson, David~J Fleet, and
  Pushmeet Kohli.
\newblock Efficient non-greedy optimization of decision trees.
\newblock In {\em Advances in Neural Information Processing Systems},
  volume~28. Curran Associates, Inc., 2015.

\bibitem{norouzi2015co2}
Mohammad Norouzi, Maxwell~D Collins, David~J Fleet, and Pushmeet Kohli.
\newblock {CO2} forest: Improved random forest by continuous optimization of
  oblique splits.
\newblock {\em arXiv preprint arXiv:1506.06155}, 2015.

\bibitem{pmlr-v119-zharmagambetov20a}
Arman Zharmagambetov and Miguel Carreira-Perpinan.
\newblock Smaller, more accurate regression forests using tree alternating
  optimization.
\newblock In {\em Proceedings of the 37th International Conference on Machine
  Learning}, volume 119, pages 11398--11408. PMLR, 13--18 Jul 2020.

\end{thebibliography}
\bibliographystyle{unsrt}

\newpage
\appendix
 
\startcontents[appendix]
\section*{Appendix}
\printcontents[appendix]{}{1}{}
\newpage

\section{Pure Block Generation}
\label{case}
\subsection{Illustrative Case of Pure Block Generatio}

Figure \ref{LHT_f} demonstrates how four choices of the threshold \(c\) generate pure subblocks under the condition \(N_{\max} \geq \gamma\):
\begin{itemize}
    \item When \(N_1 = N_{\max}\), \(c = \min \text{NFS}\), producing a pure target block \(B_1\).
    \item When \(N_2 = N_{\max}\), \(c = \max \text{NFS} + \delta_1\), producing a pure target block \(B_2\).
    \item When \(N_3 = N_{\max}\), \(c = \min \text{TFS}\), producing a pure non-target block \(B_1\).
    \item When \(N_4 = N_{\max}\), \(c = \max \text{TFS} + \delta_1\), producing a pure non-target block \(B_2\).
\end{itemize}

These figures visually depict the practical outcomes of the \(c\) selection strategy.

\begin{figure*}[htbp]
\vskip 0.1in
	\centering    
	\subfigure[The case where $c=\min \text{NFS}$ is illustrated when $N_1 = N_{\max}$ and $N_{max}\geq\gamma$] 
	{
		\begin{minipage}[t]{0.5\linewidth}
			\centering          
			\includegraphics[width=1\textwidth]{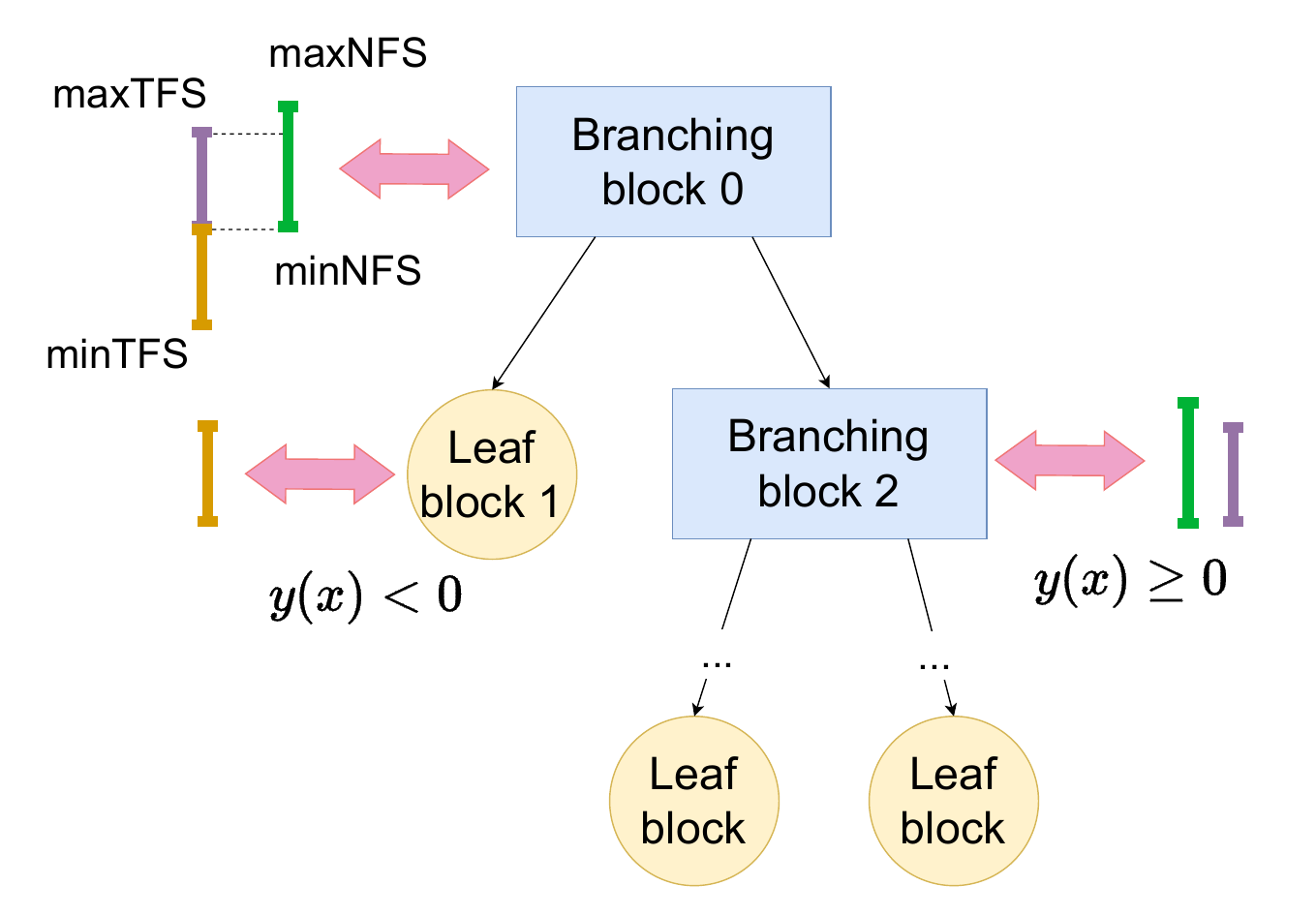}   
      \label{threshold1}
		\end{minipage}
	}
	\subfigure[The case where $c=\max \text{NFS} + \delta_1$ is illustrated when $N_2 = N_{\max}$ and $N_{max}\geq\gamma$] 
	{
		\begin{minipage}[t]{0.46\linewidth}
			\centering      
			\includegraphics[width=1\textwidth]{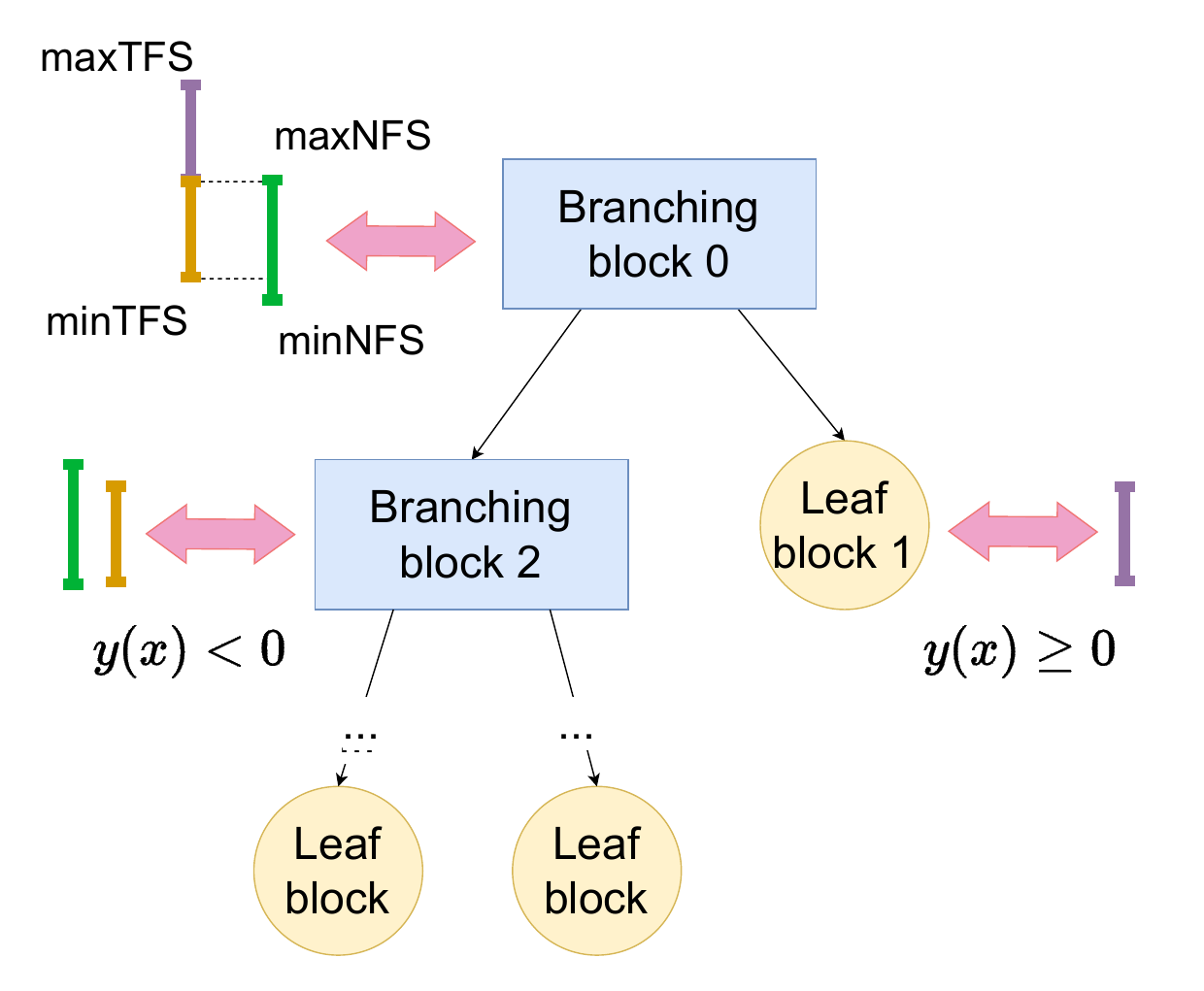}   
          \label{threshold2}
		\end{minipage}
	}
    	\subfigure[The case where $c=\min \text{TFS} $ is illustrated when $N_3 = N_{\max}$ and $N_{max}\geq\gamma$] 
	{
		\begin{minipage}[t]{0.5\linewidth}
			\centering          
			\includegraphics[width=1\textwidth]{fig//threshold3.pdf}   
      \label{threshold3}
		\end{minipage}
	}
	\subfigure[The case where $c=\max \text{TFS} + \delta_1$ is illustrated when $N_4 = N_{\max}$ and $N_{max}\geq\gamma$] 
	{
		\begin{minipage}[t]{0.46\linewidth}
			\centering      
			\includegraphics[width=1\textwidth]{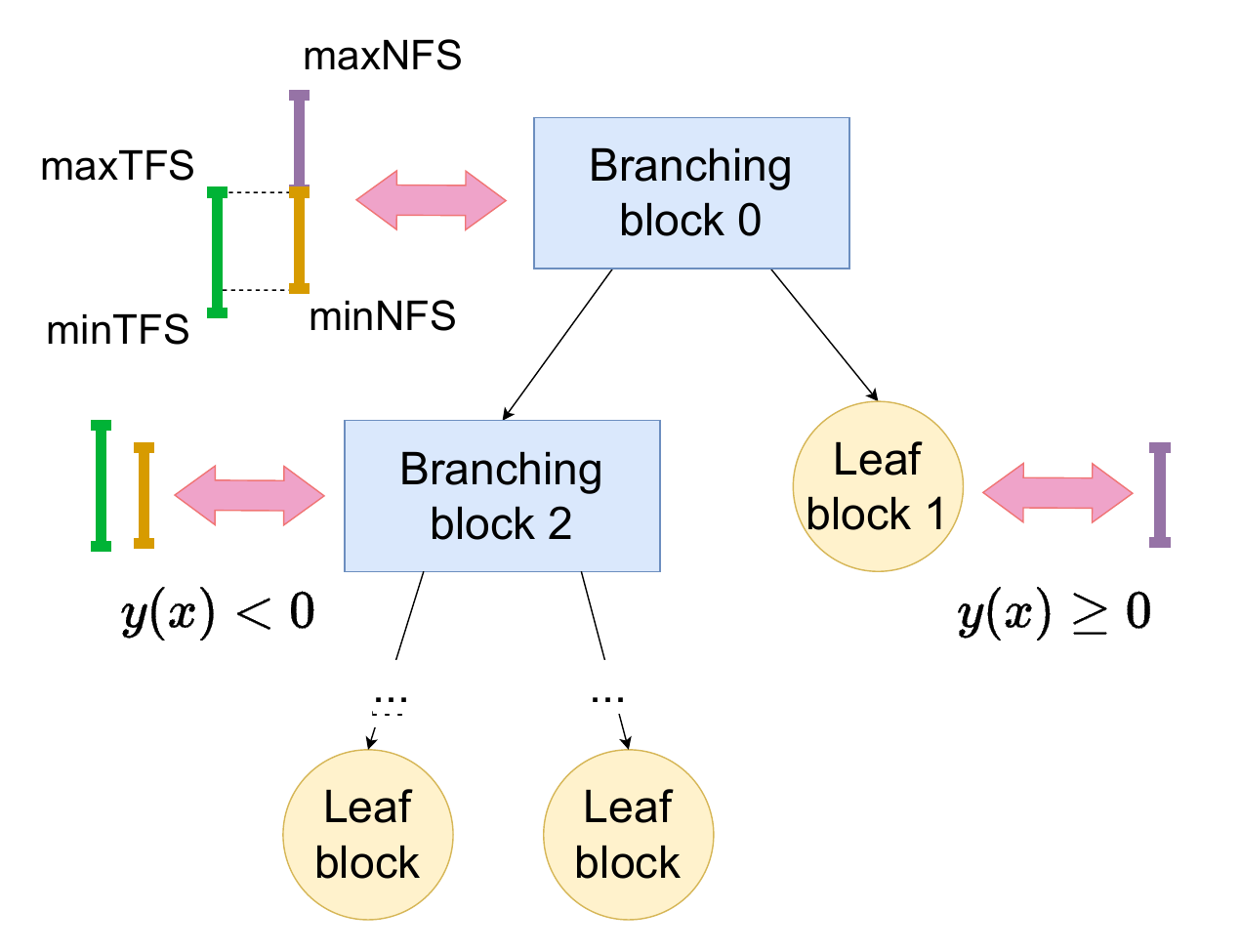}   
          \label{threshold4}
		\end{minipage}
	}
	\caption{Illustrative cases of pure block generation when \(N_{\max} \geq \gamma\). Each subfigure depicts the resulting subblocks for different \(c\) values.
} 
	\label{LHT_f}  
\end{figure*}
\vskip 0.3in

\subsection{The Role of \texorpdfstring{$\delta_1\:$}{}in Hyperplane Splitting}
The parameter \(\delta_1 > 0\) is a small positive constant used to adjust the splitting threshold in the LHT, ensuring the purity of subblocks by excluding boundary samples in maximum cases.
In the LHT, the hyperplane \(y(\boldsymbol{x}) = \text{FS}(\boldsymbol{x}) - c = 0\) divides samples into two subblocks:
\begin{itemize}
    \item \(B_1 = \{\boldsymbol{x}_j \mid \text{FS}(\boldsymbol{x}_j) < c\}\),
    \item \(B_2 = \{\boldsymbol{x}_j \mid \text{FS}(\boldsymbol{x}_j) \geq c\}\),
\end{itemize}
where \(\text{FS}(\boldsymbol{x}_j) = \sum_{i=1}^m w_i x_{ij}\) is the feature-weighted sum of sample \(\boldsymbol{x}_j\), and \(c\) is the threshold parameter. The goal is to select \(c\) such that at least one subblock is pure, containing only target or non-target samples. The choice of \(c\) is based on:
\begin{itemize}
    \item \(N_1 = \left| \{ j \mid \text{FS}(\boldsymbol{x}_j) \in \mathcal{FS}^\text{t}, \text{FS}(\boldsymbol{x}_j) < \min \text{NFS} \} \right|\): number of target samples with feature-weighted sum less than \(\min \text{NFS}\),
    \item \(N_2 = \left| \{ j \mid \text{FS}(\boldsymbol{x}_j) \in \mathcal{FS}^\text{t}, \text{FS}(\boldsymbol{x}_j) > \max \text{NFS} \} \right|\): number of target samples with feature-weighted sum greater than \(\max \text{NFS}\),
    \item \(N_3 = \left| \{ j \mid \text{FS}(\boldsymbol{x}_j) \in \mathcal{FS}^\text{nt}, \text{FS}(\boldsymbol{x}_j) < \min \text{TFS} \} \right|\): number of non-target samples with feature-weighted sum less than \(\min \text{TFS}\),
    \item \(N_4 = \left| \{ j \mid \text{FS}(\boldsymbol{x}_j) \in \mathcal{FS}^\text{nt}, \text{FS}(\boldsymbol{x}_j) > \max \text{TFS} \} \right|\): number of non-target samples with feature-weighted sum greater than \(\max \text{TFS}\),
\end{itemize}
with \(N_{\max} = \max\{N_1, N_2, N_3, N_4\}\), and:
\[
c = \begin{cases}
\min \text{NFS}, & \text{if } N_1 = N_{\max} \text{ and } N_{\max} \geq \gamma, \\
\max \text{NFS} + \delta_1, & \text{if } N_2 = N_{\max} \text{ and } N_{\max} \geq \gamma, \\
\min \text{TFS}, & \text{if } N_3 = N_{\max} \text{ and } N_{\max} \geq \gamma, \\
\max \text{TFS} + \delta_1, & \text{if } N_4 = N_{\max} \text{ and } N_{\max} \geq \gamma, \\
e, & \text{if } N_{\max} < \gamma,
\end{cases}
\]
where \(\gamma > 0\) is the minimum sample threshold, \(\delta_1 > 0\) is a small positive constant, and \(e\) is the average of extreme feature-weighted sums.

When \(c = \max \text{NFS}\) or \(c = \max \text{TFS}\) without \(\delta_1\), \(B_2\) becomes impure:
\begin{itemize}
    \item If \(c = \max \text{NFS}\), \(B_2 = \{\boldsymbol{x}_j \mid \text{FS}(\boldsymbol{x}_j) \geq \max \text{NFS}\}\) includes non-target samples with \(\text{FS}(\boldsymbol{x}_j) = \max \text{NFS}\) and target samples with \(\text{FS}(\boldsymbol{x}_j) > \max \text{NFS}\) ( \(N_2\) samples), resulting in impurity. Adding \(\delta_1\), \(c = \max \text{NFS} + \delta_1\), ensures \(B_2\) contains only target samples with \(\text{FS}(\boldsymbol{x}_j) > \max \text{NFS}\), forming a pure block. (See Figure \ref{threshold2}  for reference.)
    \item If \(c = \max \text{TFS}\), \(B_2 = \{\boldsymbol{x}_j \mid \text{FS}(\boldsymbol{x}_j) \geq \max \text{TFS}\}\) includes target samples with \(\text{FS}(\boldsymbol{x}_j) = \max \text{TFS}\) and non-target samples with \(\text{FS}(\boldsymbol{x}_j) > \max \text{TFS}\) ( \(N_4\) samples), resulting in impurity. Adding \(\delta_1\), \(c = \max \text{TFS} + \delta_1\), ensures \(B_2\) contains only non-target samples with \(\text{FS}(\boldsymbol{x}_j) > \max \text{TFS}\), forming a pure block. (See Figure \ref{threshold4}  for reference.)
\end{itemize}

In contrast, when \(c = \min \text{NFS}\) or \(c = \min \text{TFS}\), \(B_1\) is inherently pure without \(\delta_1\):
\begin{itemize}
    \item If \(c = \min \text{NFS}\), \(B_1 = \{\boldsymbol{x}_j \mid \text{FS}(\boldsymbol{x}_j) < \min \text{NFS}\}\) contains only target samples (\(N_1\) samples), as all non-target samples have \(\text{FS}(\boldsymbol{x}_j) \geq \min \text{NFS}\). (See Figure \ref{threshold1}  for reference.)
    \item If \(c = \min \text{TFS}\), \(B_1 = \{\boldsymbol{x}_j \mid \text{FS}(\boldsymbol{x}_j) < \min \text{TFS}\}\) contains only non-target samples (\(N_3\) samples), as all target samples have \(\text{FS}(\boldsymbol{x}_j) \geq \min \text{TFS}\). (See Figure \ref{threshold3}  for reference.)
\end{itemize}

Thus, \(\delta_1\) is essential in maximum cases to adjust \(c\) and exclude boundary samples, ensuring \(B_2\)'s purity, while it is unnecessary in minimum cases where \(B_1\) is naturally pure.

\section{LHT Construction and Prediction Algorithms}

\label{alg_app}
\newcommand{\E}{\mathbb{E}}
\newcommand{\FS}{\text{FS}}

This section outlines the procedures for LHT construction and prediction.

\begin{algorithm}[H] 
\caption{LHT Node Splitting and Tree Building (Recursive)}
\label{alg:lht_build_appendix} 
\begin{algorithmic}[1]

\Procedure{BuildLHTNode}{NodeData $(\boldsymbol{X}, P)$, current\_depth} 
    \State \textbf{Check Stopping Conditions:}
    \If{NodeData is pure \textbf{or} $|P| < \text{min\_samples}$}
        \State Create a \textbf{LeafNode} and \textcolor{black}{compute the membership function.}

    \EndIf

    \State \textbf{Feature Selection:}
    \State Calculate variance $V_i$; select features with $V_i > \alpha$.
    \State Calculate $\text{SD}_i$, $\overline{\text{SD}}$, $w_i$.
    \State Select final features with $|w_i| > \beta$.

    \State \textbf{Compute Feature-Weighted Sums (FS):}
    \State For each sample $\bm{x}_j$, compute $\FS(\bm{x}_j)$ using final features and weights.
    \State Get FS sets $\mathcal{FS}^\text{t}$ and $\mathcal{FS}^\text{nt}$.

    \State \textbf{Determine Optimal Split Threshold $c$:}
    \State Calculate 5 candidates ($\min \text{TFS}$, $\max \text{TFS}$, $\min \text{NFS}$, $\max \text{NFS}$, $e$).
    \State Calculate $N_1, N_2, N_3, N_4$, $N_{\max}$.
    \State Select $c$ based on $N_{\max}$ and $\gamma$ (using logic from Eq. \ref{c}).

    \State \textbf{Create Branch Node:}
    \State Create a BranchNode; store final weights $\{w_i\}$ and threshold $c$.
    \State \textbf{Split Data:}
    \State LeftData $\leftarrow \{(\bm{x}_j, p_j) \mid \FS(\bm{x}_j) < c\}$.
    \State RightData $\leftarrow \{(\bm{x}_j, p_j) \mid \FS(\bm{x}_j) \geq c\}$.

    \State \textbf{Recursively Build Subtrees:}
    \State BranchNode.left\_child $\leftarrow$ \Call{BuildLHTNode}{LeftData, current\_depth + 1}.
    \State BranchNode.right\_child $\leftarrow$ \Call{BuildLHTNode}{RightData, current\_depth + 1}.

    \State \textbf{Return} BranchNode.
\EndProcedure
\end{algorithmic}
\end{algorithm}

\begin{algorithm}[H] 
\caption{LHT Prediction}
\label{alg:lht_predict_appendix} 
\begin{algorithmic}[1]
\Procedure{PredictLHT}{TrainedLHT (root node), TestData $\bm{x}$}
    \State CurrentNode $\leftarrow$ TrainedLHT.
    \While{CurrentNode is not a LeafNode}
        \State Get weights $\{w_i\}$ and threshold $c$ from CurrentNode.
        \State Compute $\FS(\bm{x}) = \sum_{i} w_i x_i$.
        \State Compute $y(\bm{x}) = \FS(\bm{x}) - c$.
        \If{$y(\bm{x}) < 0$}
            \State CurrentNode $\leftarrow$ CurrentNode.left\_child.
        \Else
            \State CurrentNode $\leftarrow$ CurrentNode.right\_child.
        \EndIf
    \EndWhile
    \State \Comment{CurrentNode is now the reached LeafNode}
    \State Get Least Squares parameters ($a_i^*, b^*, \E[X_i], \E[P]$) from CurrentNode.
    \State Compute $\hat{p}(\bm{x}) = \sum a_i^*(x_i - \E[X_i]) + \E[P]$.
    \State Compute MembershipValue $\mu(\bm{x}) = \max\{0, \min\{\hat{p}(\bm{x}), 1\}\}$.
    \State \textbf{Return} MembershipValue ($\mu(\bm{x})$).
\EndProcedure
\end{algorithmic}
\end{algorithm}


\section{Convergence Analysis of the Block Splitting Process}
\label{sec:appendix_proof}

In this section, we establish the convergence of the proposed block splitting algorithm. We prove that each split operation results in two non-empty subblocks. \textcolor{black}{Because each step effectively reduces the size of the data by partitioning the samples into strictly smaller (non-empty) sets, this ensures the process terminates in a finite number of steps, thereby preventing infinite loops.}

\subsection{Assumptions}

We make the following standard assumptions for the analysis:
\setcounter{assumption}{0}
\begin{assumption} \label{ass:non_trivial_block1}
The block to be split contains at least two samples (\(n \geq 2\)). Furthermore, the block contains at least one sample belonging to the target class and at least one sample belonging to the non-target class.
\end{assumption}
\begin{assumption} \label{ass:splittable1}
The feature-weighted sums \(\text{FS}(\boldsymbol{x}_j)\) are distinct for at least two samples within the block. This ensures that a split is always possible.
\end{assumption}
\begin{assumption} \label{ass:threshold_gamma1}
The threshold \(\gamma\) used in the splitting criterion is a positive integer satisfying \(1 \leq \gamma \leq n\).
\end{assumption}

\subsection{Notation and Setup}

Let a block consist of \(n\) samples \(\{\boldsymbol{x}_1, \ldots, \boldsymbol{x}_n\}\), where each \(\boldsymbol{x}_j \in \mathbb{R}^m\). The feature-weighted sum for sample \(\boldsymbol{x}_j\) is defined as:
\[
\text{FS}(\boldsymbol{x}_j) = \sum_{i=1}^m w_i x_{ij},
\]
where \(w_i\) represents the weight associated with the \(i\)-th feature. The block splitting process partitions the samples based on a threshold \(c\). A sample \(\boldsymbol{x}_j\) is assigned to the first subblock if \(\text{FS}(\boldsymbol{x}_j) < c\) and to the second subblock if \(\text{FS}(\boldsymbol{x}_j) \geq c\).

Let \(\mathcal{S} = \{1, \ldots, n\}\) be the index set of samples in the block. Let \(\mathcal{S}^\text{t} \subseteq \mathcal{S}\) be the index set of target class samples and \(\mathcal{S}^\text{nt} \subseteq \mathcal{S}\) be the index set of non-target class samples. By Assumption \ref{ass:non_trivial_block1}, \(\mathcal{S}^\text{t} \neq \emptyset\) and \(\mathcal{S}^\text{nt} \neq \emptyset\). Let \(\mathcal{FS} = \{\text{FS}(\boldsymbol{x}_j) \mid j \in \mathcal{S}\}\) be the set of feature-weighted sums. We define:
\begin{itemize}
    \item \(\mathcal{FS}^\text{t} = \{\text{FS}(\boldsymbol{x}_j) \mid j \in \mathcal{S}^\text{t}\}\)
    \item \(\mathcal{FS}^\text{nt} = \{\text{FS}(\boldsymbol{x}_j) \mid j \in \mathcal{S}^\text{nt}\}\)
    \item \(\max \text{TFS} = \max(\mathcal{FS}^\text{t})\), \(\min \text{TFS} = \min(\mathcal{FS}^\text{t})\)
    \item \(\max \text{NFS} = \max(\mathcal{FS}^\text{nt})\), \(\min \text{NFS} = \min(\mathcal{FS}^\text{nt})\)
\end{itemize}
Note that these extrema exist since \(\mathcal{FS}^\text{t}\) and \(\mathcal{FS}^\text{nt}\) are non-empty finite sets.

The splitting threshold \(c\) is selected based on the following quantities:
\[
\begin{aligned}
N_1 &= | \{ j \in \mathcal{S}^\text{t} \mid \text{FS}(\boldsymbol{x}_j) < \min \text{NFS} \} |, \\
N_2 &= | \{ j \in \mathcal{S}^\text{t} \mid \text{FS}(\boldsymbol{x}_j) > \max \text{NFS} \} |, \\
N_3 &= | \{ j \in \mathcal{S}^\text{nt} \mid \text{FS}(\boldsymbol{x}_j) < \min \text{TFS} \} |, \\
N_4 &= | \{ j \in \mathcal{S}^\text{nt} \mid \text{FS}(\boldsymbol{x}_j) > \max \text{TFS} \} |.
\end{aligned}
\]
Let \(N_{\max} = \max\{N_1, N_2, N_3, N_4\}\). The threshold \(c\) is chosen as follows:
\[
c = \begin{cases}
\min \text{NFS}, & \text{if } N_1 = N_{\max} \text{ and } N_{\max} \geq \gamma, \\
\max \text{NFS} + \delta_1, & \text{if } N_2 = N_{\max} \text{ and } N_{\max} \geq \gamma, \\
\min \text{TFS}, & \text{if } N_3 = N_{\max} \text{ and } N_{\max} \geq \gamma, \\
\max \text{TFS} + \delta_1, & \text{if } N_4 = N_{\max} \text{ and } N_{\max} \geq \gamma, \\
e, & \text{if } N_{\max} < \gamma,
\end{cases}
\]
where \(e = \frac{\min \text{NFS} + \max \text{NFS} + \min \text{TFS} + \max \text{TFS}}{4}\), and \(\delta_1\) is an infinitesimally small positive number ensuring that \(c\) is strictly greater than \(\max \text{NFS}\) or \(\max \text{TFS}\) respectively, but smaller than the next distinct value in \(\mathcal{FS}\) if one exists.

\subsection{Convergence Guarantee}

We now state and prove the main theorem regarding the convergence of the splitting process.
\setcounter{theorem}{0}
\begin{theorem} \label{thm:convergence1}
Under Assumptions \ref{ass:non_trivial_block1}, \ref{ass:splittable1}, and \ref{ass:threshold_gamma1}, the block splitting process based on the threshold \(c\) defined above always partitions the block into two non-empty subblocks: \(B_1 = \{\boldsymbol{x}_j \mid \text{FS}(\boldsymbol{x}_j) < c\}\) and \(B_2 = \{\boldsymbol{x}_j \mid \text{FS}(\boldsymbol{x}_j) \geq c\}\).
\end{theorem}

\begin{proof}
We prove the theorem by considering two cases based on the value of \(c\). Let \(B_1 = \{\boldsymbol{x}_j \mid \text{FS}(\boldsymbol{x}_j) < c\}\) and \(B_2 = \{\boldsymbol{x}_j \mid \text{FS}(\boldsymbol{x}_j) \geq c\}\). We need to show that \(B_1 \neq \emptyset\) and \(B_2 \neq \emptyset\).

\textbf{Case 1: \( c \neq e \)}
In this case, \(N_{\max} \geq \gamma \geq 1\). This implies that at least one of \(N_1, N_2, N_3, N_4\) is positive. We analyze the sub-cases based on the choice of \(c\):

\begin{itemize}
    \item \textbf{Sub-case 1.1: \(c = \min \text{NFS}\)}.
    Here, \(N_1 = N_{\max} \geq \gamma\). By definition, there are \(N_1 \geq 1\) target samples \(\boldsymbol{x}_j\) such that \(\text{FS}(\boldsymbol{x}_j) < \min \text{NFS} = c\). These samples belong to \(B_1\). Thus, \(B_1 \neq \emptyset\).
    Furthermore, all non-target samples \(\boldsymbol{x}_k\) satisfy \(\text{FS}(\boldsymbol{x}_k) \geq \min \text{NFS} = c\). Since the original block contains at least one non-target sample (Assumption \ref{ass:non_trivial_block1}), these non-target samples belong to \(B_2\). Thus, \(B_2 \neq \emptyset\). ( Refer to the visual illustration in Figure \ref{threshold1} for better intuition.)

    \item \textbf{Sub-case 1.2: \(c = \max \text{NFS} + \delta_1\)}.
    Here, \(N_2 = N_{\max} \geq \gamma\). By definition, there are \(N_2 \geq 1\) target samples \(\boldsymbol{x}_j\) such that \(\text{FS}(\boldsymbol{x}_j) > \max \text{NFS}\). Since \(\delta_1\) is infinitesimally small, \(\text{FS}(\boldsymbol{x}_j) \geq \max \text{NFS} + \delta_1 = c\) holds for these samples (assuming no \(\text{FS}(\boldsymbol{x}_j)\) falls exactly at \(\max \text{NFS}\); if it does, the strict inequality \(>\) in the definition of \(N_2\) applies. The use of \(\delta_1\) ensures the split occurs correctly even with potential ties if handled properly, but the core idea is that these \(N_2\) samples end up in \(B_2\)). Thus, \(B_2 \neq \emptyset\).
    All non-target samples \(\boldsymbol{x}_k\) satisfy \(\text{FS}(\boldsymbol{x}_k) \leq \max \text{NFS} < \max \text{NFS} + \delta_1 = c\). Since the block contains at least one non-target sample, these samples belong to \(B_1\). Thus, \(B_1 \neq \emptyset\). ( Refer to the visual illustration in Figure \ref{threshold2} for better intuition.)

    \item \textbf{Sub-case 1.3: \(c = \min \text{TFS}\)}.
    Here, \(N_3 = N_{\max} \geq \gamma\). By definition, there are \(N_3 \geq 1\) non-target samples \(\boldsymbol{x}_j\) such that \(\text{FS}(\boldsymbol{x}_j) < \min \text{TFS} = c\). These samples belong to \(B_1\). Thus, \(B_1 \neq \emptyset\).
    All target samples \(\boldsymbol{x}_k\) satisfy \(\text{FS}(\boldsymbol{x}_k) \geq \min \text{TFS} = c\). Since the original block contains at least one target sample (Assumption \ref{ass:non_trivial_block1}), these target samples belong to \(B_2\). Thus, \(B_2 \neq \emptyset\). ( Refer to the visual illustration in Figure \ref{threshold3} for better intuition.)

    \item \textbf{Sub-case 1.4: \(c = \max \text{TFS} + \delta_1\)}.
    Here, \(N_4 = N_{\max} \geq \gamma\). By definition, there are \(N_4 \geq 1\) non-target samples \(\boldsymbol{x}_j\) such that \(\text{FS}(\boldsymbol{x}_j) > \max \text{TFS}\). As in Sub-case 1.2, these samples satisfy \(\text{FS}(\boldsymbol{x}_j) \geq \max \text{TFS} + \delta_1 = c\) and belong to \(B_2\). Thus, \(B_2 \neq \emptyset\).
    All target samples \(\boldsymbol{x}_k\) satisfy \(\text{FS}(\boldsymbol{x}_k) \leq \max \text{TFS} < \max \text{TFS} + \delta_1 = c\). Since the block contains at least one target sample, these samples belong to \(B_1\). Thus, \(B_1 \neq \emptyset\). ( Refer to the visual illustration in Figure \ref{threshold4} for better intuition.)
\end{itemize}
In all sub-cases where \(c \neq e\), both \(B_1\) and \(B_2\) are non-empty.

\textbf{Case 2: \( c = e \)}
This case occurs when \(N_{\max} < \gamma\). The threshold \(c = e\) is the average of the four extreme feature-weighted sum values (\(\min \text{NFS}, \max \text{NFS}, \min \text{TFS}, \max \text{TFS}\)). Let \(\text{FS}_{\min} = \min(\mathcal{FS})\) and \(\text{FS}_{\max} = \max(\mathcal{FS})\). By definition, \(\min \text{NFS} \geq \text{FS}_{\min}\), \(\max \text{NFS} \leq \text{FS}_{\max}\), \(\min \text{TFS} \geq \text{FS}_{\min}\), and \(\max \text{TFS} \leq \text{FS}_{\max}\). Therefore,
\[
\text{FS}_{\min} \leq \frac{\text{FS}_{\min} + \text{FS}_{\min} + \text{FS}_{\min} + \text{FS}_{\min}}{4} \leq e \leq \frac{\text{FS}_{\max} + \text{FS}_{\max} + \text{FS}_{\max} + \text{FS}_{\max}}{4} = \text{FS}_{\max}.
\]
So, \(e\) lies within the range \([\text{FS}_{\min}, \text{FS}_{\max}]\).
By Assumption \ref{ass:splittable1}, there exist at least two samples with distinct feature-weighted sums. Thus, \(\text{FS}_{\min} < \text{FS}_{\max}\) (since \(n \geq 2\)).
If all \(\text{FS}(\boldsymbol{x}_j)\) were equal, the block would not be splittable, contradicting Assumption \ref{ass:splittable1}.
Therefore, there must be at least one sample \(\boldsymbol{x}_k\) such that \(\text{FS}(\boldsymbol{x}_k) = \text{FS}_{\min}\) and at least one sample \(\boldsymbol{x}_l\) such that \(\text{FS}(\boldsymbol{x}_l) = \text{FS}_{\max}\).

Can \(c = e\) be equal to \(\text{FS}_{\min}\) or \(\text{FS}_{\max}\)?
If \(e = \text{FS}_{\min}\), then \(\min \text{NFS} = \max \text{NFS} = \min \text{TFS} = \max \text{TFS} = \text{FS}_{\min}\). This implies all samples have the same feature-weighted sum \(\text{FS}_{\min}\), contradicting Assumption \ref{ass:splittable1}.
Similarly, \(e = \text{FS}_{\max}\) leads to a contradiction.
Therefore, \(\text{FS}_{\min} < e < \text{FS}_{\max}\).

Since \(\text{FS}_{\min} < e\), the sample(s) \(\boldsymbol{x}_k\) with \(\text{FS}(\boldsymbol{x}_k) = \text{FS}_{\min}\) satisfy \(\text{FS}(\boldsymbol{x}_k) < e\). Thus, \(B_1 \neq \emptyset\).
Since \(e < \text{FS}_{\max}\), the sample(s) \(\boldsymbol{x}_l\) with \(\text{FS}(\boldsymbol{x}_l) = \text{FS}_{\max}\) satisfy \(\text{FS}(\boldsymbol{x}_l) > e\), which implies \(\text{FS}(\boldsymbol{x}_l) \geq e\). Thus, \(B_2 \neq \emptyset\).

In both Case 1 and Case 2, the splitting process results in two non-empty subblocks \(B_1\) and \(B_2\). Since each split reduces the size of the block being considered (as \(|B_1| \geq 1\), \(|B_2| \geq 1\), and \(|B_1| + |B_2| = n\)), and the process stops when blocks meet certain criteria (e.g., size or purity), the overall block splitting process is guaranteed to converge, i.e., given a finite number of initial samples, the overall block splitting process is guaranteed to terminate within a finite number of steps, thereby eliminating the possibility of infinite loops.
\end{proof}

\newcommand{\Cov}{\text{Cov}} 
\newcommand{\Var}{\text{Var}} 

\section{Computational Complexity Analysis}
\label{sec:appendix_complexity_proof}

We analyze the time complexity of constructing the LHT.

\begin{theorem}
Given a dataset with \(n\) samples and \(m\) features, the time complexity required to build an LHT up to a depth \(d\) is \(O(mnd)\).
\end{theorem}

\begin{proof}
The construction of the LHT involves recursively partitioning the data at each node. Let's analyze the computational cost at a single node and then aggregate it over the tree structure.

Consider a node \(N\) containing \(n_{\text{node}}\) samples. The computations performed at this node before splitting are:

\begin{enumerate}
    \item \textbf{Feature Selection (Expectation/Weight Calculation):}
        This step involves calculating statistics for each feature based on the \(n_{\text{node}}\) samples within the current node to determine the splitting hyperplane. This includes calculating the mean (\(\mathbb{E}[X_i]\), \(\mathbb{E}[X_i^2]\), \(\mathbb{E}[X_i^{\text{t}}]\), \(\mathbb{E}[X_i^{\text{nt}}]\)) for each feature, the difference between target and non-target means (\(\text{SD}_i\)), and the feature weights (\(w_i\)). For each feature, calculating these statistics requires iterating through the \(n_{\text{node}}\) samples.  Therefore, calculating the required statistics for *each* of the *m* features requires \(O(m \cdot n_{\text{node}})\) time.  Finding the maximum absolute $SD_i$ (i.e.  $\overline{SD}$) and calculating the weights $w_i$ takes O(m) time. Thus, the feature selection step requires a total time of  \(O(m \cdot n_{\text{node}})\).

    \item \textbf{Feature-Weighted Sum Calculation:}
        For each of the \(n_{\text{node}}\) samples, we calculate the feature-weighted sum (\(\text{FS}(\boldsymbol{x})\)). This requires iterating through its \(m\) features, involving \(m\) multiplications and \(m-1\) additions.  Thus, each feature-weighted sum calculation takes \(O(m)\) time. Since there are \(n_{\text{node}}\) samples at the node, the total time is \(O(m \cdot n_{\text{node}})\).

    \item \textbf{Hyperplane Parameter (c) Selection:} This step involves calculating \(N_1, N_2, N_3, N_4\), \(e\) and determining the best value for the hyperplane parameter \(c\). This requires computing FS(\(\boldsymbol{x}_j\)) for each of the \(n_{\text{node}}\) samples and comparing them to \(\min \text{NFS}\), \(\max \text{NFS}\), \(\min \text{TFS}\), and \(\max \text{TFS}\). This comparison and counting takes \(O(n_{\text{node}})\) time for each of the four \(N_i\) values, so the total time is \(O(n_{\text{node}})\).  Calculating \(e\) from  \(\min \text{NFS}\), \(\max \text{NFS}\), \(\min \text{TFS}\), and \(\max \text{TFS}\) takes \(O(1)\) time. Selecting the appropriate \(c\) also takes constant time. So, the time for parameter selection is  \(O(n_{\text{node}})\).

    \item \textbf{Data Partitioning:}
        After selecting the hyperplane, we need to partition the \(n_{\text{node}}\) samples into two sub-nodes based on the hyperplane equation \(y(\boldsymbol{x}) = \text{FS}(\boldsymbol{x}) - c = 0\). For each sample, evaluating the condition \(y(\boldsymbol{x}) < 0\) takes \(O(1)\) time (since \(\text{FS}(\boldsymbol{x})\) has already been computed). Therefore, partitioning all \(n_{\text{node}}\) samples takes \(O(n_{\text{node}})\) time.

\end{enumerate}

Combining these steps, the total work performed at a single node is dominated by the feature selection and feature-weighted sum calculations, each resulting in a complexity of \(O(m \cdot n_{\text{node}})\) per node. Therefore, the overall complexity at a single node is \(O(m \cdot n_{\text{node}})\).

Now, let's consider the complexity across the entire tree up to depth \(d\). A common way to analyze tree algorithms is level by level. At any given level \(k\) (where \(0 \le k < d\)), let the nodes be \(N_{k,1}, N_{k,2}, \ldots\). Let \(n_{k,i}\) be the number of samples in node \(N_{k,i}\). The crucial observation is that each sample from the original dataset belongs to exactly one node at level \(k\). Therefore, the total number of samples across all nodes at level \(k\) is \(\sum_{i} n_{k,i} \le n\). Note that strict inequality can occur if some nodes are leaf nodes, and their processing ends.

The total work performed across all nodes at level \(k\) is the sum of the work done at each node:
\[
\sum_{i} O(m \cdot n_{k,i}) = O\left(m \sum_{i} n_{k,i}\right) \le O(m n).
\]

This means that the total computational cost for processing all samples across one entire level of the tree is bounded by \(O(mn)\).

Since the tree is built up to a depth \(d\), there are \(d\) such levels (from level 0 to level \(d-1\)) where these computations are performed. Therefore, the total time complexity for building the LHT is the cost per level multiplied by the number of levels:

\[
d \times O(mn) = O(mnd).
\]

This completes the proof.
\end{proof}

\section{Derivation Details of Membership Functions}
\label{ap1}

We are given a labeled dataset with $n$ samples. The target vector is $P=[p_1, \dots, p_n]^\top$, where $p_i \in \{0, 1\}$. We have $m$ features, and the data for the $i$-th feature is represented by the vector $X_i = [x_{1i}, \dots, x_{ni}]^\top$. The objective is to construct a function $\mu(\bm{x})$ that outputs a membership degree in the range [0, 1] for a new input sample $\bm{x}=[x_1, \dots, x_m]$.

\textcolor{black}{To simplify computation, we adopt an approach based on univariate linear regression to construct the membership function, rather than performing a full multivariate regression. This method analyzes each feature independently and combines the results to form a multivariate prediction model, offering a computationally efficient approximation.}

First, we determine coefficients based on univariate features. We analyze the linear relationship between each feature $X_i$ and the target $P$ separately. For each feature $X_i$ ($i=1, \dots, m$), we perform a linear regression. The goal is to find the slope $a_i$ and intercept $c_i$ that minimize the sum of squared errors between the observed target values $p_k$ and the values predicted by the linear model $a_i x_{ki} + c_i$. This involves solving the following optimization problem:
\[
    \min_{a_i, c_i} L(a_i, c_i) = \min_{a_i, c_i} \sum_{k=1}^n (p_k - (a_i x_{ki} + c_i))^2
\]
According to the principle of ordinary least squares (OLS), the optimal slope $a_i^*$ that minimizes this objective function is given by the ratio of the sample covariance between the feature and the target to the sample variance of the feature:
\begin{equation}
    a_i^* = \frac{\Cov(X_i, P)}{\Var(X_i)} = \frac{\sum_{k=1}^n (x_{ki} - \E[X_i])(p_k - \E[P])}{\sum_{k=1}^n (x_{ki} - \E[X_i])^2}
    \label{eq:ai_individual_fit_ols}
\end{equation}
Here, $\E[\cdot]$ denotes the sample mean (e.g., $\E[X_i] = \frac{1}{n}\sum_k x_{ki}$), $\Cov(X_i, P)$ is the sample covariance between $X_i$ and $P$, and $\Var(X_i)$ is the sample variance of $X_i$.  This step yields a coefficient $a_i^*$ for each feature, determined independently based on its univariate linear relationship with the target.

Next, we combine these feature contributions and determine a global intercept. We use the individually determined coefficients $a_1^*, \dots, a_m^*$ to form a combined linear predictor for a multi-feature input $\bm{x}$:
\[
    \hat{p}(\bm{x}) = a_1^* x_1 + a_2^* x_2 + \dots + a_m^* x_m + b^*
\]
where $b^*$ is the intercept term yet to be determined. We set $b^*$ by imposing a global condition: the average prediction of the model over the training data must equal the average value of the target labels in the training data. That is, we require $\E[\hat{P}] = \E[P]$, where $\hat{P}$ is the vector of predictions for the training samples. This leads to:
\[
    \E[\hat{P}] = \E\left[ \sum_{i=1}^m a_i^* X_i + b^* \right] = \sum_{i=1}^m a_i^* \E[X_i] + b^*
\]
Setting $\E[\hat{P}] = \E[P]$ and solving for $b^*$:
\begin{equation}
    b^* = \E[P] - \sum_{i=1}^m a_i^* \E[X_i] \label{eq:b_star_mean_align_ols}
\end{equation}
This choice of $b^*$ ensures that the overall level of the model's predictions matches the average level of the target data.

Substituting the expression for $b^*$ back into the prediction model yields the final prediction model for a new sample $\bm{x} = [x_1, \dots, x_m]$:
\[
    \hat{p}(\bm{x}) = \sum_{i=1}^m a_i^* x_i + \left( \E[P] - \sum_{j=1}^m a_j^* \E[X_j] \right)
\]
This can be rewritten in the centered form:
\begin{equation}
  \hat{p}(\bm{x}) = \sum_{i=1}^m a_i^* (x_i - \E[X_i]) + \E[P] \label{eq:p_hat_centered_combined_ols}
\end{equation}
where $a_i^*$ is computed using \eqref{eq:ai_individual_fit_ols}, and $\E[X_i]$, $\E[P]$ are the sample means from the training data.

Finally, we generate the membership degree. Since the linear prediction $\hat{p}(\bm{x})$ can fall outside the [0, 1] range required for a membership degree, we clip it to this interval to obtain the final membership function $\mu(\bm{x})$:
\begin{equation}
    \mu(\bm{x}) = \max\{0, \min\{\hat{p}(\bm{x}), 1\}\} \label{eq:membership_clipped_combined_ols}
\end{equation}
This $\mu(\bm{x})$ represents the constructed membership degree for the input sample $\bm{x}$.

\section{Note on Computational Complexity of Membership Functions}
\label{note}
To clarify the computation of membership functions within the leaf blocks and its impact on the overall complexity, we provide the following detailed explanation:

\begin{itemize}
    \item \textbf{Computation Location:} membership functions are computed within the leaf blocks of the LHT. Each leaf block independently constructs its own Membership Function based on its contained samples using least squares fitting.
    
    \item \textbf{Complexity for a Single Leaf Block:} For a leaf block containing \(n_{\text{leaf}}\) samples, the time complexity to compute its Membership Function is \(O(m \cdot n_{\text{leaf}})\), where \(m\) is the number of features. Specifically:
    \begin{itemize}
        \item Calculating the coefficient \(a_i^* = \frac{\Cov(X_i, P)}{\Var(X_i)}\) for each feature \(i\) requires \(O(n_{\text{leaf}})\) time, including computations for sample means (\(\E[X_i]\), \(\E[P]\)), covariance (\(\Cov(X_i, P)\)), and variance (\(\Var(X_i)\)).
        \item Performing this for all \(m\) features results in a total time of \(O(m \cdot n_{\text{leaf}})\).
        \item Computing the global intercept \(b^* = \E[P] - \sum_{i=1}^m a_i^* \E[X_i]\) takes \(O(m)\) time, which is negligible compared to \(O(m \cdot n_{\text{leaf}})\).
    \end{itemize}
    
    \item \textbf{Total Complexity for All Leaf Blocks:} Suppose the LHT has \(L\) leaf blocks, with each leaf block containing \(n_1, n_2, \dots, n_L\) samples, respectively, and \(\sum_{i=1}^L n_i = n\) (since each sample belongs to exactly one leaf block). The total time complexity for constructing membership functions across all leaf blocks is:
    \[
    O\left( m \cdot \sum_{i=1}^L n_i \right) = O(m \cdot n)
    \]
    
    \item \textbf{Impact on Overall Complexity:} The construction complexity of the LHT is \(O(mnd)\), where \(d\) is the maximum depth of the tree, as established in Appendix \ref{sec:appendix_complexity_proof}. Typically, \(d \geq 1\) (e.g., \(d = O(\log n)\)), so \(O(mnd) \geq O(mn)\). The computation of membership functions in the leaf blocks, performed as the final step, incurs an additional cost of \(O(mn)\), which is subsumed by \(O(mnd)\). Thus, the overall time complexity remains \(O(mnd)\), consistent with the original analysis.
\end{itemize}

 Therefore, the computation of membership functions within the leaf blocks does not alter the overall complexity of \(O(mnd)\).

\section{ Universal Approximation Capability of LHTs}
\label{sec:universal_approximation}
We prove that LHTs can universally approximate continuous functions on a compact set \( K \subset \mathbb{R}^m \).

\begin{theorem}
Let \( K \subset \mathbb{R}^m \) be a compact set and \( g: K \to [0,1] \) be a continuous function. Then, for any \( \epsilon > 0 \), there exists an LHT-defined function \( f_{\text{LHT}}: K \to [0,1] \) such that:
\[ \sup_{\mathbf{x} \in K} |f_{\text{LHT}}(\mathbf{x}) - g(\mathbf{x})| < \epsilon \]
\end{theorem}

\begin{proof}
Since \( K \) is compact, \( g \) is uniformly continuous on \( K \). Thus, for any \( \epsilon > 0 \), there exists \( \delta > 0 \) such that for all \( \mathbf{x}, \mathbf{y} \in K \), if \( \|\mathbf{x} - \mathbf{y}\| < \delta \), then:
\[ |g(\mathbf{x}) - g(\mathbf{y})| < \frac{\epsilon}{2} \]

Our proof constructs an LHT by recursively partitioning \( K \) into leaf regions, controlling the oscillation of \( g \) within each region. Specifically, to establish universal approximation, i.e., \(\sup_{\mathbf{x} \in K} |f_{\text{LHT}}(\mathbf{x}) - g(\mathbf{x})| < \epsilon\), our LHT construction ensures that the function \(g\) has an oscillation (i.e., \(\sup_{\mathbf{x}, \mathbf{y} \in R} |g(\mathbf{x}) - g(\mathbf{y})|\)) of less than \(\epsilon/2\) within each leaf region.
By the uniform continuity of \(g\) on the compact set \(K\), this oscillation condition is met if the diameter of every leaf region is reduced below a threshold \(\delta\).
Thus, our proof employs a conservative yet provable strategy: we show that LHTs can refine regions until all leaf diameters are less than \(\delta\), ensuring the required local approximation quality for \( g \).

Let \( D \) be the diameter of \( K \). In the splitting process of LHT, each hyperplane split divides a region into two subregions, reducing the diameter by a factor \(\lambda_i\) at the \(i\)-th split, where \(0 < \lambda_i < 1\). Since \(\lambda_i\) may vary depending on the hyperplane's position, we define \(\lambda_{\text{max}} = \max_i \lambda_i\) as the largest reduction factor across all possible splits. 
In the worst case, after \(d\) splits, the diameter of a leaf region is at most \(D \lambda_{\text{max}}^d\). To ensure this is less than \(\delta\), we require:
\[ D \lambda_{\text{max}}^d < \delta \implies d > \frac{\log \left( \frac{D}{\delta} \right)}{\log \left( \frac{1}{\lambda_{\text{max}}} \right)} \]
Taking \( d = \left\lceil \frac{\log \left( \frac{D}{\delta} \right)}{\log \left( \frac{1}{\lambda_{\text{max}}} \right)} \right\rceil + 1 \) ensures all leaf regions have diameters less than \( \delta \).

For each leaf region \( R \), select a set of points \( \{\mathbf{x}_i\}_{i=1}^n \subset R \cap K \), and use least squares to fit an affine function \( L_R(\mathbf{x}) = \mathbf{a}_R^T \mathbf{x} + b_R \) that approximates \( g(\mathbf{x}_i) \). 

Since \( \text{diam}(R) < \delta \), for any \( \mathbf{x}, \mathbf{y} \in R \cap K \), we have \( \|\mathbf{x} - \mathbf{y}\| < \delta \), which by uniform continuity implies \( |g(\mathbf{x}) - g(\mathbf{y})| < \epsilon/2 \). This means the oscillation of \( g \) within the small region \( R \cap K \) is less than \( \epsilon/2 \). This local near-constancy allows \( g \) to be well approximated by an affine function within \( R \cap K \).
Specifically, there exists an affine function \( L_R(\mathbf{x}) = \mathbf{a}_R^T \mathbf{x} + b_R \) such that for all \( \mathbf{x} \in R \cap K \):
\[ |L_R(\mathbf{x}) - g(\mathbf{x})| < \frac{\epsilon}{2} \]
Indeed, even a simple constant function, e.g., \( L_R(\mathbf{x}) = g(\mathbf{x}_R) \) for some fixed \( \mathbf{x}_R \in R \cap K \), would satisfy \( |L_R(\mathbf{x}) - g(\mathbf{x})| = |g(\mathbf{x}_R) - g(\mathbf{x})| < \epsilon/2 \) due to the small oscillation. An affine function via local least squares fitting, offering more degrees of freedom, can certainly achieve this required bound.

Define the LHT function in each leaf region \( R \) as:
\[ f_{\text{LHT}}(\mathbf{x})  = \max\{0, \min\{L_R(\mathbf{x}), 1\}\} \]

For any \( \mathbf{x} \in K \), let \( R \) be the leaf region containing \( \mathbf{x} \). Since \( |L_R(\mathbf{x}) - g(\mathbf{x})| < \frac{\epsilon}{2} \), i.e., \( L_R(\mathbf{x}) - \frac{\epsilon}{2}<g(\mathbf{x}) < L_R(\mathbf{x}) + \frac{\epsilon}{2}  \), and \( g(\mathbf{x}) \in [0,1] \), we analyze:
\begin{itemize}
    \item If \( 0 \leq L_R(\mathbf{x}) \leq 1 \), then \( f_{\text{LHT}}(\mathbf{x}) = L_R(\mathbf{x}) \), so \( |f_{\text{LHT}}(\mathbf{x}) - g(\mathbf{x})| < \frac{\epsilon}{2} \).
    \item If \( L_R(\mathbf{x}) < 0 \), then \(f_{\text{LHT}}(\mathbf{x}) = 0 \), and since \( g(\mathbf{x}) \geq 0 \) and \( g(\mathbf{x}) < L_R(\mathbf{x}) + \frac{\epsilon}{2}<\frac{\epsilon}{2} \), \( |f_{\text{LHT}}(\mathbf{x}) - g(\mathbf{x})| < \frac{\epsilon}{2} \).
    \item If \( L_R(\mathbf{x}) > 1 \), then \( f_{\text{LHT}}(\mathbf{x}) = 1 \), and since \( g(\mathbf{x}) \leq 1 \) and \( 1-g(\mathbf{x}) < 1-L_R(\mathbf{x}) + \frac{\epsilon}{2}< \frac{\epsilon}{2} \), \( |f_{\text{LHT}}(\mathbf{x}) - g(\mathbf{x})| < \frac{\epsilon}{2} \).
\end{itemize}
Thus, \( |f_{\text{LHT}}(\mathbf{x}) - g(\mathbf{x})| < \frac{\epsilon}{2} < \epsilon \) for all \( \mathbf{x} \in K \), implying:
\[ \sup_{\mathbf{x} \in K} |f_{\text{LHT}}(\mathbf{x}) - g(\mathbf{x})| < \epsilon \]
This establishes the universal approximation capability of LHTs.
\end{proof}

\section{Feature Importance of the LHT Branching Blocks in the Wine Dataset}
\label{ap2}

The Wine dataset comprises 178 samples, each characterized by 13 features representing chemical and physical properties of wines. These features are: Alcohol, Malic acid, Ash, Alcalinity of ash, Magnesium, Total phenols, Flavanoids, Nonflavanoid phenols, Proanthocyanins, Color intensity, Hue, OD280/OD315 of diluted wines, and Proline. This dataset poses a three-class classification problem, making it suitable for evaluating the feature importance of the LHT in multi-class settings.

To illustrate the contribution of each feature in the LHT branch, we use the Wine dataset as an example. 
80\% of the data is used for training, and 20\% is used for testing. For each class, a separate LHT is trained, resulting in a total of three LHTs for classification. The results of the LHT classification, under the condition of fixed $\alpha=0$ and varying $\beta$, are shown in Table \ref{sample-table1}.
The value of $\beta$ can influence the growth of LHT in different ways. For the Wine dataset, selecting fewer but more discriminative features (i.e., using a larger $\beta$) does not necessarily lead to a decrease in test accuracy. 
\begin{table}[h]
	\caption{Classification accuracy of the LHT on the Wine dataset for varying $\beta$.}
	\label{sample-table1}
	\centering
	\begin{center}
				\begin{tabular}{{@{}c@{\hspace{10pt}}c@{\hspace{10pt}}c@{\hspace{10pt}}c@{\hspace{10pt}}c@{\hspace{10pt}}c@{\hspace{10pt}}c@{}}}
					\toprule
					$\beta=0$ & $\beta=0.25$& $\beta=0.5$& $\beta=0.75$& $\beta=1$\\
					\midrule
				    97.6 $\pm$ 2.8 & 97.8 $\pm$ 1.6 &96.8 $\pm$ 2.4&93.0 $\pm$ 3.2&94.1 $\pm$ 4.2\\
					\bottomrule
				\end{tabular}
	\end{center}
\end{table}

The LHT structures of the three classes in the Wine dataset are shown in Figure \ref{LHT_s0}, with the left side corresponding to the case where $\beta=0$, and the right side to the case where $\beta=0.25$.
For class 0, the LHT structures are the same for both $\beta=0$ and $\beta=0.25$. For class 1, the number of blocks decreases when $\beta=0.25$, while for class 2, the number of blocks increases when $\beta=0.25$.

\begin{figure*}[ht]
\begin{center}
\centerline{\includegraphics[width=0.9\textwidth]{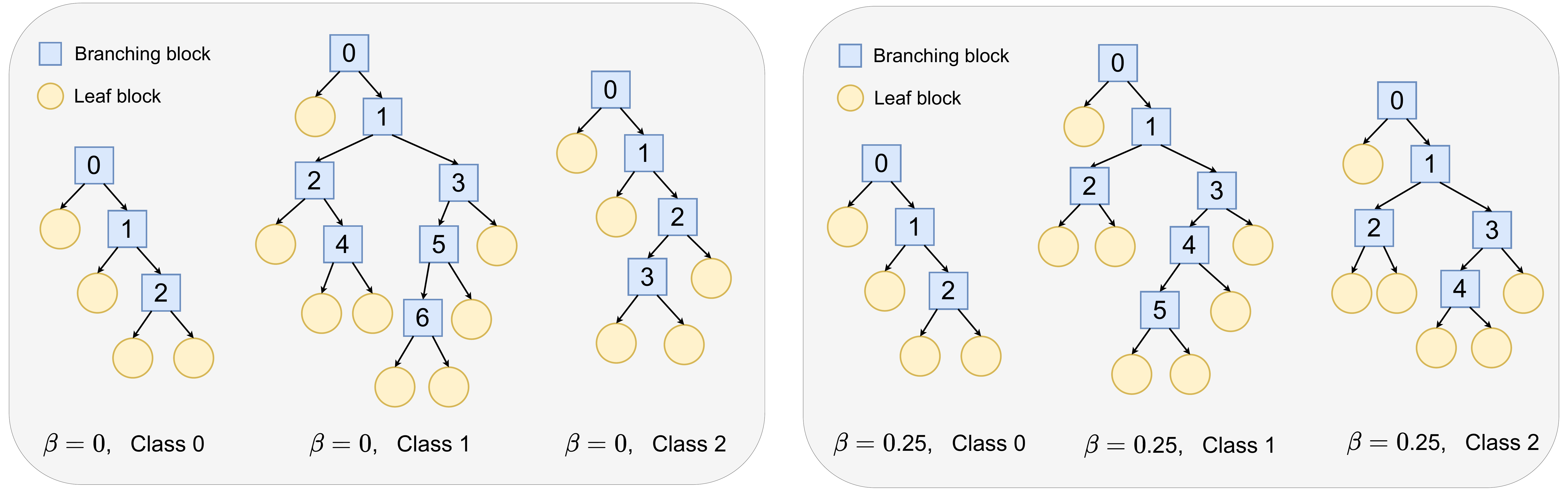}}
\caption{The LHT structures of the three classes in the Wine dataset are shown, with the left side corresponding to the case where $\beta=0$, and the right side to the case where $\beta=0.25$.}
\label{LHT_s0}
\end{center}
\end{figure*}
\textcolor{black}{
LHT's interpretability stems from its transparent structure and statistically-driven feature weights (\(w_i\)), which define splitting hyperplanes.}
Figure \ref{LHT_f14} ($\beta=0$) and \ref{LHT_02514} ($\beta=0.25$) demonstrate the contribution of each feature to the classification of the branching blocks. In the split of branching block 0 for class 0, feature 12 (Proline) is the most influential. For branching block 1 for class 0, feature 0 (Alcohol) plays the most significant role, whereas feature 2 (Ash) dominates in the split of branching block 2.
\textcolor{black}{Such insights enable domain experts to validate model decisions or optimize wine formulations.}
The situations for class 1 and class 2 can be analyzed using similar visualizations.
Due to the transparency of the LHT structure, the contribution of each feature at every branching step is clearly visible.

\begin{figure*}[htbp]
\begin{center}
\centerline{\includegraphics[width=0.9\textwidth]{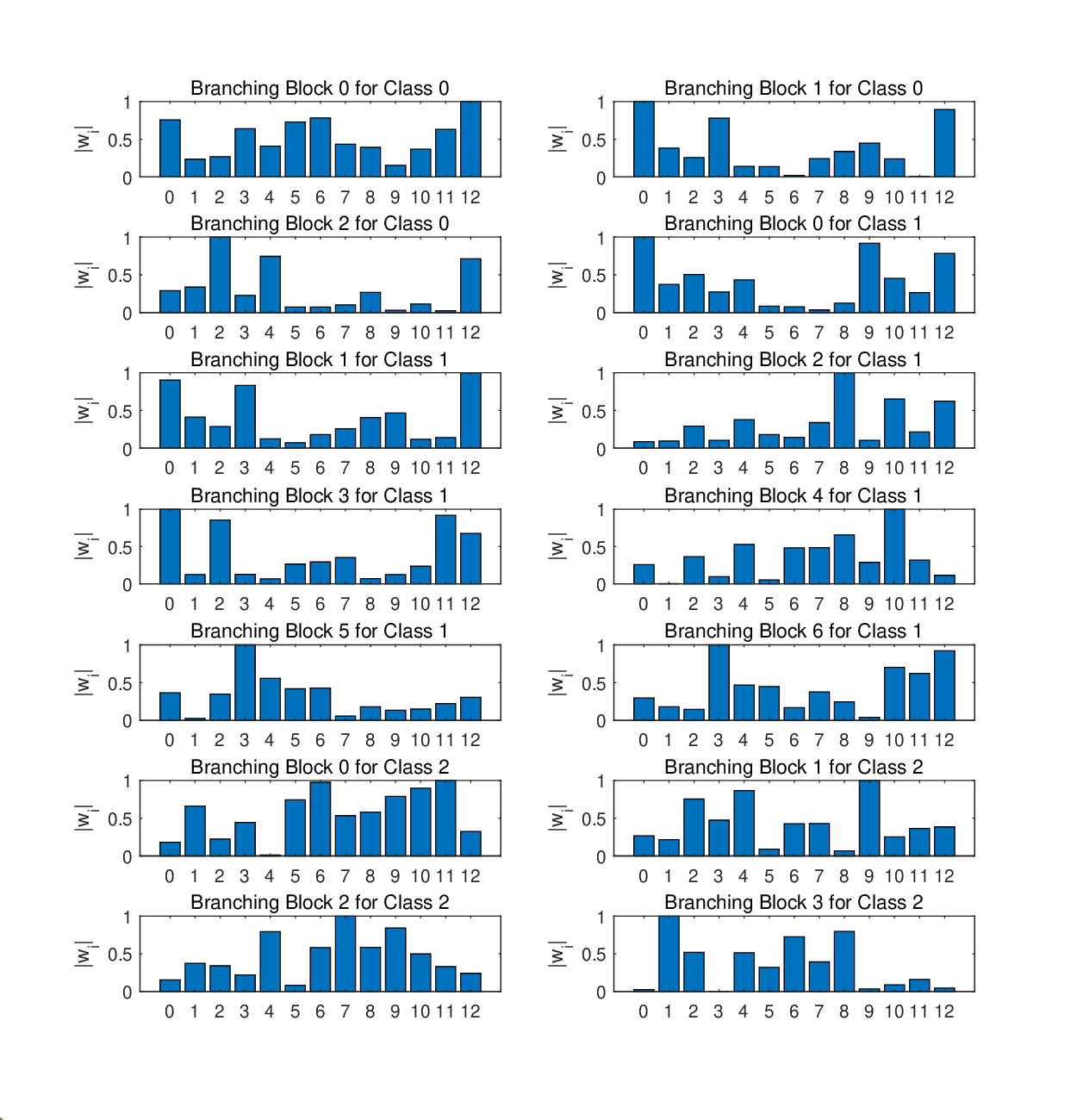}}
\caption{Visualization of the feature weights for each branching block of the three LHTs corresponding to the three classes in the Wine dataset ($\beta=0$).}
\label{LHT_f14}
\end{center}
\end{figure*}
\clearpage

\newpage
\begin{figure*}[!htbp]
\begin{center}
\centerline{\includegraphics[width=0.9\textwidth]{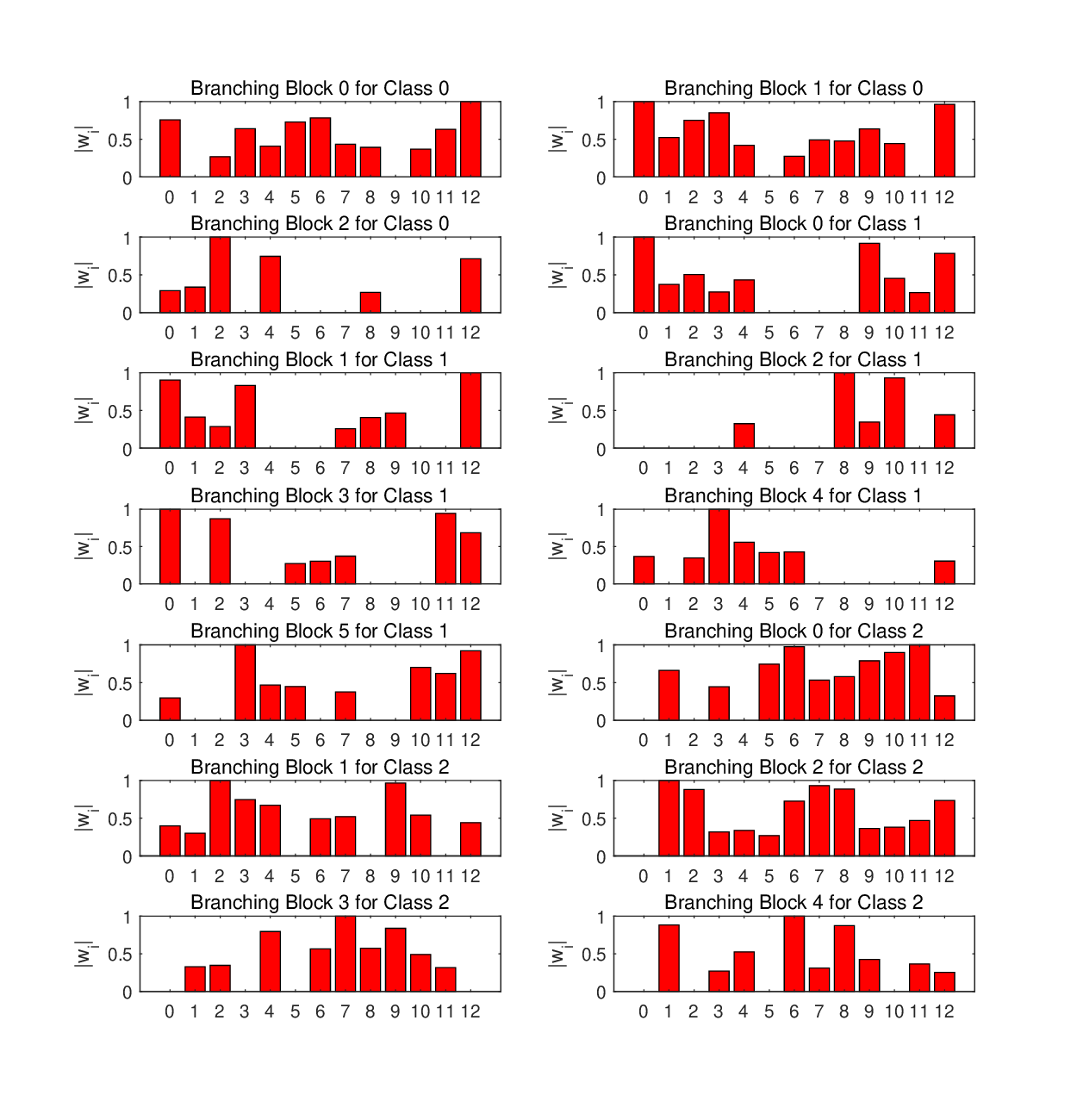}}
\caption{Visualization of the feature weights for each branching block of the three LHTs corresponding to the three classes in the Wine dataset ($\beta=0.25$).}
\label{LHT_02514}
\end{center}
\end{figure*}

\clearpage

\newpage

\section{Hyperparameters of the Experiment and Dataset Information}
\label{hyperpara}
We provide the hyperparameters for the experimental setup here.
Table \ref{h2} provides the hyperparameter information for LHT across 30 test datasets. For the Wine dataset, $\beta$ is chosen to be 0.25. For other datasets where $\beta^{\prime}$ values are specified, the LH forest method is used. The forest rate refers to the proportion of training samples randomly selected to construct the LHT. For more details about LH forest, refer to Section \ref{LHf}.
In addition, except for the MNIST dataset, $\alpha$ is set to 0 for all other datasets, and all datasets have been normalized. We used the 784-dimensional MNIST dataset, which was not normalized, with $\alpha$ set to 900. All experiments have the same \text{min\_samples} and $\gamma$ values.

Table \ref{hyperparametre} presents the hyperparameter settings for RF, XGBoost, CatBoost, and LightGBM. The hyperparameters were manually selected across multiple datasets, taking into account dataset-specific characteristics such as sample size and feature dimensionality. We explored values for max depth (3–15), learning rate (0.01–0.1), and the number of estimators (20–400), and selected the final configuration based on the highest accuracy obtained through five-fold cross-validation on the training data.

All the datasets used for testing are from UCI\footnote{\url{https://archive.ics.uci.edu/datasets }} and LIBSVM\footnote{\url{https://www.csie.ntu.edu.tw/~cjlin/libsvmtools/datasets/}}. For datasets with predefined training and testing splits, we adhere to the provided splits. For datasets lacking predefined splits, we partition the data into training and testing sets using an 80:20 ratio.

To estimate error bars, we employ distinct strategies based on the dataset's split configuration. For datasets without predefined splits, we perform random splits at an 80:20 training-to-testing ratio and conduct a minimum of ten independent experiments to compute the error bars. For datasets with predefined splits, we generate error bars by bootstrapping the test set (maintaining the original test set size) and performing at least ten repeated experiments.

\begin{table*}[!htbp]
	\caption{Hyperparameters for LHT are specified, and '–' indicates that LH forests are not used.}
	\label{h2}
	\centering
	\begin{center}
				\begin{tabular}{@{}l@{\hspace{8pt}}c@{\hspace{10pt}}c@{\hspace{10pt}}c@{\hspace{10pt}}c@{}}
					\toprule
					Dataset &$\gamma$, \text{min\_samples} &  $\beta^{\prime}(\beta)$ & Tree Num./Class &Forest Rate\\
					\midrule
                    Protein &   80& 0&100 & 60\%\\
SatImages &  5& 0&30 & 80\%\\
Segment &   2& 0& -& -\\
Pendigits &  2& 0&20 &80\% \\
Connect-4 &   7& 0& 40& 60\%\\
SensIT &  6 &0 & 100&60\% \\
Letter &  2 & 0.8&50 & 80\%\\
Balance Scale & 3& 0& -& -\\
Blood Trans. & 5& 0&150&30\%\\
Acute Inflam. 1 &   2&0 &- & -\\
Acute Inflam. 2 &  2&0 &- & -\\
Car Evaluation &  3 &0 & 100& 80\%\\
Breast Cancer &  2 & 0& 50&80\% \\
Avila Bible &  5&0 & 10& 80\% \\
Wine-Red &  2 &0 & -& -\\
Wine-White &   2 &0 & -& -\\
Dry Bean & 2 &0 & -& - \\
Climate Crashes &  2& 0&- & -\\
Conn. Sonar & 2&0 & 10& 80\%\\
Optical Recog. & 4& 0& 10& 80\% \\
					Wine& 2 & 0.25 & -&-\\
					Seeds & 2&$0.3$&10&100\% \\
					WDBC & 2 & 0 & -&-\\
					Banknote & 2 & 0.8 &2&100\%\\
					Rice & 3 & 0 &-&-\\
					Spambase & 4 & 0 &25&80\%\\
					EEG & 2 & 0.3 &25&100\% \\
					MAGIC & 3 &   0 &50&100\%\\
					SKIN & 2 &     0&50&80\%\\
                    Mnist & 6 &     0.8&20&100\%\\
					\bottomrule
				\end{tabular}
	\end{center}
\end{table*}

\begin{table}[!htbp]
    \caption{Hyperparameters for RF, XGBoost, CatBoost, and LightGBM.}
    \label{hyperparametre}
    \centering
    \begin{center}
        \scriptsize
                \begin{tabular}{@{}l@{\hspace{13pt}}l@{\hspace{10pt}}c@{\hspace{10pt}}c@{\hspace{10pt}}c@{}}
                    \toprule
                    Dataset & Method &  Max Depth & Learning Rate & Tree Num. \\
                    \midrule
                    \multirow{4}{*}{Protein}&  RF&  15 &-&150\\
                       &  XGBoost&  5  &0.1&200\\
                       &  CatBoost&  15 &0.1&100\\
                       &  LightGBM&  15 &0.1&37 \\
                    \midrule
                    \multirow{4}{*}{SatImages}&  RF&  15 &-&50\\
                       &  XGBoost&   5 &0.1&130\\
                       &  CatBoost&  7 &0.1&385\\
                       &  LightGBM&  15 &0.1& 72\\
                    \midrule
                    \multirow{4}{*}{Segment}&  RF&   15&-&100\\
                       &  XGBoost&   10 &0.1&200\\
                       &  CatBoost&  5 &0.1&200\\
                       &  LightGBM&   5&0.1&37 \\
                    \midrule
                    \multirow{4}{*}{Pendigits}&  RF&  15 &-&150\\
                       &  XGBoost&  5  &0.1&300\\
                       &  CatBoost&  10 &0.1&200\\
                       &  LightGBM&  5&0.1&106 \\
                    \midrule
                    \multirow{4}{*}{Connect-4}&  RF& 15 &-&150\\
                       &  XGBoost&  10&0.1&200\\
                       &  CatBoost&  10&0.1&200\\
                       &  LightGBM&  10&0.1&110 \\
                    \midrule
                    \multirow{4}{*}{MNIST}&  RF&  15&-&150\\
                       &  XGBoost& 15 &0.1&200\\
                       &  CatBoost&  10&0.1&200\\
                       &  LightGBM& 10 &0.1&138 \\
                    \midrule
                    \multirow{4}{*}{SensIT}&  RF&15 &-&100\\
                       &  XGBoost&  10&0.1&200\\
                       &  CatBoost& 10 &0.1&200\\
                       &  LightGBM& 10 &0.05&200 \\
                    \midrule
                    \multirow{4}{*}{Letter}&  RF& 15 &-&150\\
                       &  XGBoost&  10&0.1&200\\
                       &  CatBoost&  10&0.1&200\\
                       &  LightGBM& 10&0.1&100 \\
                    \midrule
                    \multirow{4}{*}{Seeds} &  RF&  6 & - &20\\
                       &  XGBoost&  7 &0.1&50\\
                       &  CatBoost&  7 &0.1 &70\\
                       &  LightGBM&  6& 0.1 &50 \\
                    \midrule
                    \multirow{4}{*}{WDBC} &  RF&  6 &-&20\\
                       &  XGBoost&  6  &0.1&50\\
                       &  CatBoost&  8 &0.05&100\\
                       &  LightGBM&  7 &0.05&100 \\
                    \midrule
                    \multirow{4}{*}{ Banknote} &  RF& 6 &-&25\\
                       &  XGBoost&  6&0.1&50\\
                       &  CatBoost&  8 &0.1&70\\
                       &  LightGBM&  6  &0.1&50\\
                    \midrule
                    \multirow{4}{*}{Rice} &  RF&  7 &-&20\\
                       &  XGBoost&  6 &0.1&60\\
                       &  CatBoost& 9 &0.1&80\\
                       &  LightGBM&6 &0.1&50 \\
                    \midrule
                    \multirow{4}{*}{Spambase} &  RF&  10&-&30\\
                       &  XGBoost&  9 &0.1&50\\
                       &  CatBoost&  8 &0.1&60\\
                       &  LightGBM&9  &0.1&50\\
                    \midrule
                    \multirow{4}{*}{EEG} &  RF& 15 &-&50\\
                      &  XGBoost&  10&0.1&150\\
                       &  CatBoost&  15 &0.1&150\\
                       &  LightGBM&  15&0.1&200\\
                    \midrule
                    \multirow{4}{*}{MAGIC} &  RF &  6&-&50\\
                       &  XGBoost & 10&0.1&50\\
                       &  CatBoost & 10 &0.1&100\\
                       &  LightGBM & 10&0.1&50\\
                    \midrule
                    \multirow{4}{*}{SKIN} &  RF&  10&-&100\\
                       &  XGBoost& 6 &0.1& 100\\
                       &  CatBoost& 10&0.1&100\\
                       &  LightGBM&  6&0.1&100\\
                    \bottomrule
                \end{tabular}
    \end{center}
\end{table}

\clearpage

\newpage

\end{document}